\documentclass{article}
\usepackage{iclr2027_conference,times}
\usepackage[T1]{fontenc}
\usepackage[utf8]{inputenc}
\usepackage{microtype}
\usepackage{amsmath,amssymb,amsthm}
\usepackage{booktabs,multirow,graphicx,tabularx,pdflscape}
\usepackage{subcaption}
\usepackage{xcolor}
\usepackage{hyperref} 
\usepackage{xurl,float,algorithm,algpseudocode}
\usepackage{tikz}
\usetikzlibrary{arrows.meta,positioning,calc}
\usepackage[capitalize,noabbrev]{cleveref}
\graphicspath{{figures/}}

\DeclareCaptionLabelFormat{mysubfig}
  {Figure~\thefigure\alph{subfigure}}

\usepackage{hyperref}
\usepackage{url}

\usepackage{siunitx}
\usepackage{graphicx}
\graphicspath{{source/}}

\usepackage{todonotes}
\usepackage{mathtools}

\usepackage{amsthm}

\newtheorem{theorem}{Theorem}[section]
\newtheorem{lemma}[theorem]{Lemma}

\usepackage{wrapfig}

\usepackage{titlesec}

\titlespacing*{\section}
  {0pt}    
  {1.0ex}  
  {0.5ex}  

\titlespacing*{\subsection}
  {0pt}
  {0.5ex}
  {0.3ex}

\titleformat{\paragraph}[runin]
  {\normalfont\normalsize\bfseries}
  {\theparagraph}
  {1em}
  {}

\titlespacing*{\paragraph}
  {0pt}    
  {0.3ex}  
  {0.6em}  

\usepackage{caption}
\usepackage{algorithm}
\usepackage{algpseudocode}

\usepackage{xspace}
\makeatletter
\DeclareRobustCommand\onedot{\futurelet\@let@token\@onedot}
\def\@onedot{\ifx\@let@token.\else.\null\fi\xspace}

\newcommand{\eg}{{e.g}\onedot} 
\newcommand{\ie}{{i.e}\onedot}

\makeatother

\DeclareMathOperator*{\argmin}{\operatorname{argmin}}

\DeclareMathOperator*{\Tr}{\operatorname{Tr}}

\newcommand{\high}{\texttt{high}\xspace}
\newcommand{\low}{\texttt{low}\xspace}

\newcommand{\MSE}{\operatorname{MSE}\xspace}
\newcommand{\hatMSE}{\widehat{\MSE}}

\usepackage[framemethod=TikZ]{mdframed}
\definecolor{algobackground}{HTML}{F8F9FA}
\definecolor{algoborder}{HTML}{2B6CB0}

\mdfdefinestyle{algorithmbox}{
  backgroundcolor=algobackground,
  linecolor=algoborder,
  topline=false,
  rightline=true,
  bottomline=false,
  leftline=true,
  linewidth=3pt,
  innerleftmargin=12pt,
  innerrightmargin=12pt,
  innertopmargin=8pt,
  innerbottommargin=8pt,
  skipabove=\topsep,
  skipbelow=\topsep
}

\newcommand{\methodname}{Q-PACE\xspace}
\newcommand{\lastlayers}{NVIDIA\xspace}

\title{Q-PACE: Dynamic Precision Allocation for Quantization-Aware Training}

\author{Alexandra Volkova\thanks{Equal contribution.}\\
ISTA \\ 
\texttt{avolkova@ista.ac.at}\\
\And
Matin Ansaripour\footnotemark[1] \\
EPFL, ISTA \\
\And
Erik Schultheis\\
ISTA \\
\And
Christoph H. Lampert\\
ISTA \\
\And
Dan Alistarh\\
ISTA 
}

\iclrfinalcopy   

\begin{document}
\maketitle

\lhead{Preprint}

\maketitle

\begin{abstract}
Quantization-aware training (QAT) leverages lower-precision arithmetic to reduce the cost of LLM deployment, but aggressive quantization degrades final model performance. A common remedy is mixed-precision training, in which high precision is assigned to some of the layers to maintain performance while keeping the cost constrained. This approach then requires \textit{precision assignments} for model layers during training. We provide a new approach, called Q-PACE, consisting of a second-order sensitivity model that predicts the loss increase as a sum of quantization noise MSE weighted by per-layer curvature coefficients. During training, we periodically re-compute these coefficients using perturbations across layers, and re-assign precision. Pretraining and supervised fine-tuning experiments on LLMs of up to 4B parameters show that Q-PACE consistently improves over existing mixed-precision training recipes, and achieves comparable loss at substantially lower total memory budgets. We further find that quantization sensitivity is highly predictable by depth and layer type, and its stability during training allows for infrequent, cheap recalibration.

\end{abstract}

\section{Introduction}
\label{sec:intro}

Quantization-aware training (QAT) methods~\citep{courbariaux2016binaryconnecttrainingdeepneural, Jacob_2018_CVPR} use low-precision arithmetic to lower the cost of both \emph{deploying} LLMs, as they result in lower-bitwidth accurate models, and \emph{training} them, as the matrix multiplications during training can be mapped to lower precision. 
Recent recipes make end-to-end NVFP4 and MXFP4 pretraining viable by leveraging techniques such as block (micro) scaling, stochastic rounding, rotations, and unbiased gradient estimators~\citep{castro2025quartet,chen2025tetrajetv2,nvidia2026pretraininglargelanguagemodels}. However, there is still a clear loss gap between lower-precision approaches and BF16 and MXFP8 training, and uniform four-bit training can be lossy or even become unstable over long horizons~\citep{nvidia2026pretraininglargelanguagemodels}.

It is known that not every network layer needs the same precision, e.g.~\citet{malinovskii2025higgs}. Moreover, modern microscaling formats allow matrix products across various precisions, or even mixed-precision operands~\citep{mxspec}, so the bitwidth of each linear layer can be chosen separately, assuming a global budget constraint. The key question, which we consider here, is how to choose the layer-wise bitwidths that maximize accuracy while satisfying a precision budget. Current approaches either keep a fixed set of layers, typically the last few, in higher precision~\citep{nvidia2026pretraininglargelanguagemodels}, or rank layers by a local proxy, such as the product of gradient norm and quantization error~\citep{snip_quant}. It is not clear whether these heuristics can be improved. 

In this paper, we show that the second-order loss model as a sum of per-layer errors weighted by sensitivity coefficients accurately predicts loss increase due to quantization \textit{during training} (Sec.~\ref{sec:method}). This gives a training-time analogue of the ``linearity theorem'' model of \citet{malinovskii2025higgs}, previously used for one-shot, post-training quantization.

We call our approach~\textbf{Q-PACE}, \textbf{P}recision \textbf{A}llocation via \textbf{C}urvature \textbf{E}stimation for \textbf{Q}uantization-aware training, and apply it at training iterates. Estimated per-layer sensitivity coefficients depend only on the model weights and are reasonably stable across iterations, so \methodname re-estimates them periodically during training: once they are determined, we solve a knapsack problem to determine ``optimal'' per-layer bit-widths. This approach only requires forward passes for coefficient estimation, and has negligible memory overhead and marginal computational overhead. 

We evaluate \methodname on pretraining experiments for models between 350M and 4B parameters, and on supervised fine-tuning of Qwen3 models, against state-of-the-art baselines~\citep{nvidia2026pretraininglargelanguagemodels, snip_quant}, at budgets between 5.0 and 6.5 bits per weight. Our findings are as follows:
\begin{itemize}
    \item \methodname consistently improves over existing mixed-precision training strategies across model scales, precision budgets, and format options. In pretraining, it achieves the lowest final loss in every setting we evaluate, recovering up to $81\%$ of the loss gap between uniform NVFP4 and MXFP8 training; the same trend holds for supervised fine-tuning.
    \item \methodname provides a substantially better quality--precision trade-off. At a fixed precision budget, \methodname produces better final loss than both SNIP and the \lastlayers heuristic, while at matched loss it requires fewer bits per model parameter (Figure~\ref{fig:ranking-and-pareto-500m}).
    \item A detailed analysis of the produced allocations shows that quantization sensitivity is highly structured across the network: it depends on both layer type and depth, and its profile evolves over the course of training. However, these changes occur much slower than parameter updates. As a result, only a small number of reallocations (on the order of 10 over a full training run) is sufficient, keeping the overhead of \methodname low in practice.
\end{itemize}

\begin{figure}[t]
    \includegraphics[width=0.49\linewidth]{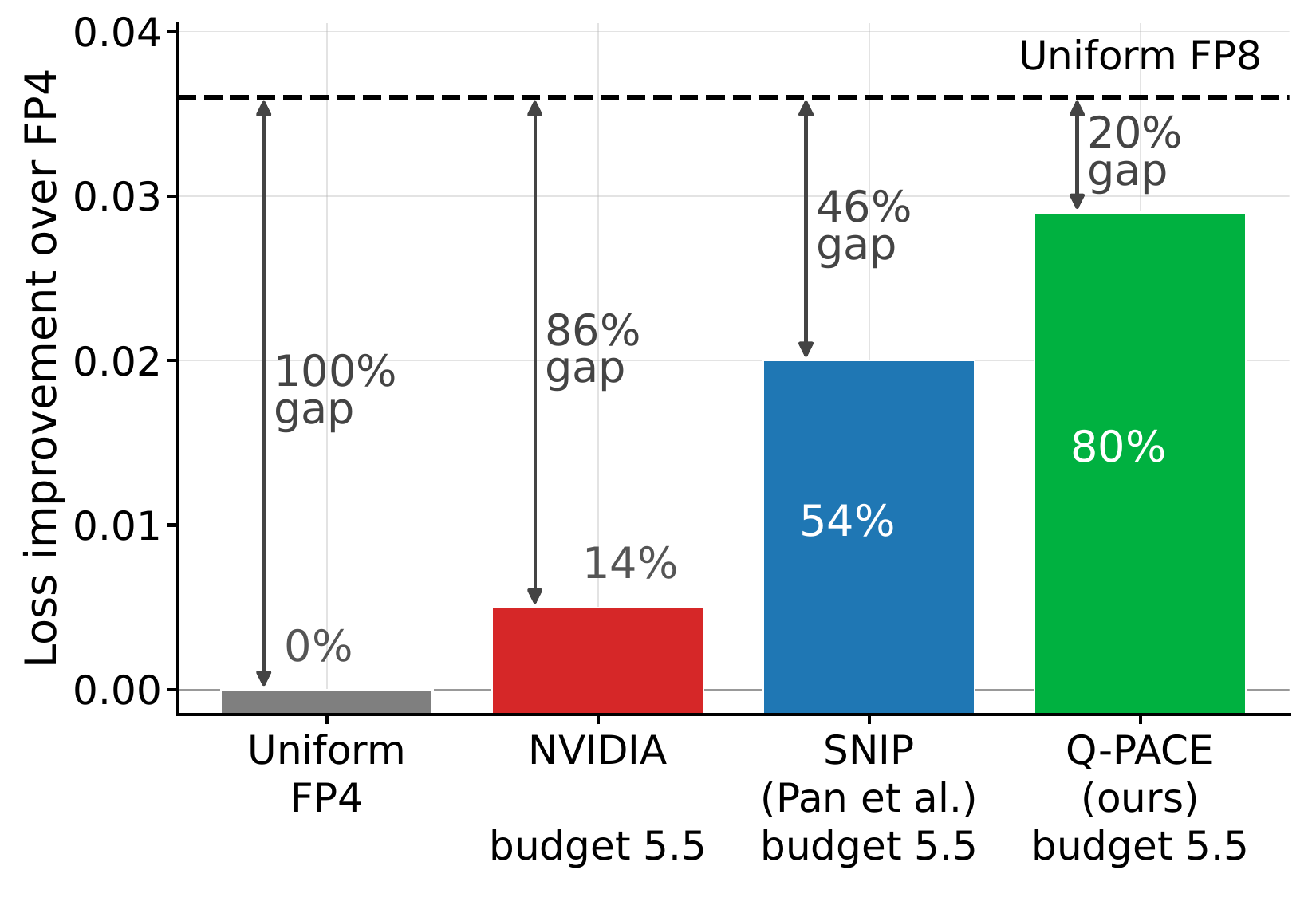} \hfill
    \includegraphics[width=0.49\linewidth]{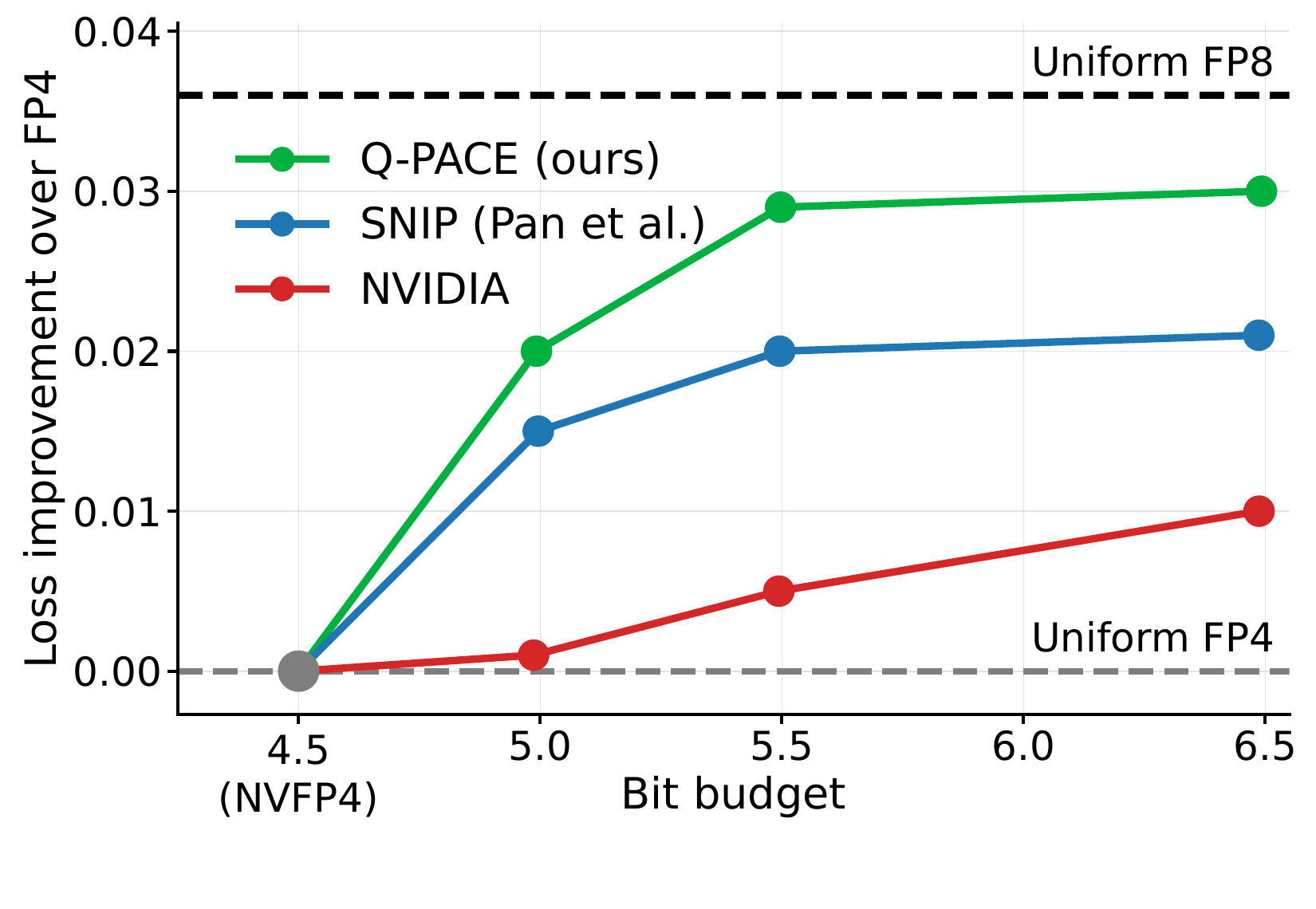}
    \caption{Final loss improvement for a 500M model trained with mixed precision across methods and budgets. We improve over \emph{last layers in BF16}~\citep{nvidia2026pretraininglargelanguagemodels} and \emph{SNIP}~\citep{snip_quant}.}
    \label{fig:ranking-and-pareto-500m}
\end{figure}

All in all, our results show that layer-wise precision can be assigned from a sensitivity model estimated using only forward passes, rather than from more complex error proxies. Doing so can recover a large fraction of the accuracy gap to uniform higher-precision training, at a small fraction of the training cost.

\section{The Problem: QAT with Adaptive Mixed-Precision}
Our goal is to choose the quantization levels of $n$ layers  
in a network during quantization-aware training. 
We concentrate on linear layers, \ie fully-connected layers or attention projections. 
For each of the full-precision weight tensors $W=(W_i)_{i=1}^n$, we want to select a corresponding \emph{bitwidth}, 
$b_i$, from a set of candidate bitwidths, denoted as $\beta$. 

We use the term \emph{bitwidth} as it is common in the literature. Technically, the selection is among \emph{formats} used for matrix multiplication operations. We distinguish between formats even if they nominally have the same number of bits but use a different value representation, such as FP8 E8M0 and FP8 E4M3.
We denote the quantization of a tensor $W_i$ by $Q(W_i, b_i)$. 
Any choice of quantization bitwidths will influence the model's overall quality, 
\ie the achieved \emph{loss}, as well as its resource \emph{cost}, such 
as storage and compute.
Typically, there is a tradeoff between those: a larger bitwidth results 
in lower loss, but higher training and inference cost.

For a sequence $B = (b_1,\dots,b_n)\in\beta^n$, called \emph{bitwidth allocation}, 
we denote its loss by $L(W, B) := L(\{ Q(W_i, b_i)  \text{ for } b_i \in B\})$, and its resource cost by $C(W, B)$. We assume the cost to be additively linear, $C(W,B) := \sum_{i=1}^n c_i(b_i)$, where $c_i(b_i) = d(W_i) \kappa(b_i)$, $d$ is a function of the size of the weight tensor $W_i$ (but not of its actual values), and $\kappa$ is a characteristic cost of a bitwidth. 

With the goal to constraint memory at inference time, we take $d(W_i)=:d_i$ as the number of elements in a tensor $W_i$, and $\kappa(b_i)$ as the number of bits required to store an element, including the format's block scales (4.5 bits for NVFP4; 6.25 and 8.25 bits for MXFP6 and MXFP8; 16 bits for BF16), which we write as $b_i$ again, meaning that $C(W,B) = \sum_{i=1}^n d_i b_i$ measures the total model size in bits. 
Note, however, that the framework we describe here would also apply to other choices of $d(W_i)$ and $\kappa(b_i)$.

Ultimately, for a model with set of weights $W$, the problem of \emph{optimal bitwidth allocation} is minimization of the loss under a given maximum model size constraint, $C_0$. Formally, we want to find  
\begin{equation}
    \label{eq:our-allocation-problem}
    B^\ast = \argmin_{B=(b_1,\dots,b_n)}\, L(W, B) \text{ subject to } \sum_i b_i d_i \leq C_0.
\end{equation}

For post-training quantization, the bitwidth allocation only has to be chosen once. In this work, however, we focus on QAT, where the choice of quantization has a direct influence on the quality of learned weights, and one cannot assume that the best quantization strategy is constant over the course of training~\citep{liu2025paretoqimprovingscalinglaws}.

Consequently, we are interested in finding a \emph{sequence of allocations} $B^1,\dots,B^T$, where $T$ is the number of training steps, that ultimately achieves the lowest final loss while adhering to the cost budget at each iteration. We use a superscript to index training steps.
The new precision allocation at each optimization step then affects both the loss calculation in the forward pass and the errors in the gradient computation in the backward pass.

The final set of weights, $W^T$, which determines the final loss, is affected by all previous allocations. Then the optimal sequence of allocations during training should be 
\begin{equation}
    (B^1,\dots,B^T)^{\ast} \coloneqq \!\!\!\argmin_{\text{all trajectories } (B^t)}\!\!\! L(W^T (B^1, \dots,  B^{T-1} ), B^T) \text{ subject to } \max_{t=1,\dots,T}\!\! C(B_t) \leq C_0.
    \label{eq:over-time-allocation-problem}
\end{equation}

Due to the complex structure with interactions between layers as well as time steps, a direct optimization of Eq.~\eqref{eq:over-time-allocation-problem} is intractable. 
In the following section, we discuss approximations to $L(W, B)$ that make the optimization computationally feasible. 

\section{Method}
\label{sec:method}

For simplicity of exposition, we first discuss a setup where only two bitwidth choices are available: high-precision, \high, and low-precision, \low, \eg BF16 and FP4, respectively. 

For a bitwidth allocation $B\in\{\low,\high\}^n$, we denote the quantization error (perturbation) of quantizing a layer into \low format as $\delta W_i := Q(W_i,b_i) - W_i$. Assuming that \high is the original data format, such that $Q(W_i, \high) = W_i$, we have
$\delta W_i=0$ if $b_i=\high$, but $\delta W_i\neq 0$
otherwise. 
With $\delta W:=(\delta W_1,\dots,\delta W_n)$, a Taylor expansion of the loss function around the unquantized loss $L(W) = L_{\high}$ yields
\begin{align}
    L(Q(W,B))-L(W)  \approx
    \sum_i \Bigl\langle \frac{\partial L(W)}{\partial W_i}, \delta W_i\Bigr\rangle  + \frac{1}{2} 
    \sum_{i,j}
    \delta W_i^\top \frac{\partial^2 L(W)}{\partial^2 W_i W_j} \delta W_j.
\label{eq:taylor}
\end{align}
Following the \emph{Linearity Theorem}~\citep{malinovskii2025higgs}, the quantization error $\delta W_i$ can be modeled as an isotropic random variable with zero mean $\mathbb{E} [\delta W_i] =0$ and finite variance $\mathbb{V} [\delta W_i]= \sigma^2_i I$, $\sigma_i^2 >0$. Moreover, we assume the errors to be uncorrelated between layers, $\mathbb{E} [\delta W_i \delta W_j] = 0$ for $1\leq i<j\leq n$. 
Assuming this, we can take the expectation over $\delta W_i$ in Eq.~\eqref{eq:taylor} to obtain:
\begin{equation}
\begin{aligned}
        \mathbb{E}\Bigl[L(W+\delta W) - L(W)\Bigr] & \approx
     \frac{1}{2} \sum_i \; \Tr(H_{ii}) \, \sigma_i^2 
    =  \frac{1}{2} \sum_i \; \Tr(H_{ii})\underbrace{\frac{1}{d_i} \mathbb{E} ||\delta W_i||^2_F}_{=\MSE(Q(W_i), W_i)},
    \end{aligned}
\label{eq:MSE}
\end{equation}

\noindent where $d_i$ denotes the number of elements in the $i$-th weight matrix, and $H_{ii}$ denotes the Hessian matrix of $L$ with respect to these weights. For a formal proof and implications see App.~\ref{app:loss-approximation}.

The derivation presented above can be done for any precision $\low$; thus, we obtain, for any bitwidth allocation:
\begin{equation}
    L(W, B) \approx L_{\high} + 
    \sum_i \alpha_i \,\hatMSE(Q(W_i,B_i), W_i),
    \label{eq:higgs}
\end{equation}
where $L_{\high} = L(W)$ is the loss over full-precision weights, and $\alpha_i$ is an
estimate of $\frac{1}{2} \Tr (H_{ii})$ that can be interpreted as the 
aggregated curvature of a layer. 

\paragraph{Connection to Hutchinson trace estimation.}
For an isotropic random probe $z$ satisfying 
$\mathbb{E}[z z^\top] = I$, we have $\mathbb{E}_z\!\left[L(W+\epsilon z)-L(W)\right] \approx  \mathbb{E}_z\!\left[z^\top H_{ii} z\right] = \frac{\epsilon^2}{2}\Tr(H_{ii})$ \citep{hutchinson}. Our forward-pass calibration can be viewed as a variant of Hutchinson trace estimation: the layer-wise curvature 
$\alpha_i$ is estimated from the increase in loss induced by isotropic 
random perturbations, normalized by their mean-squared magnitude.

\paragraph{MSE estimation.}
The estimate $\hatMSE$ can be obtained empirically by measuring the statistics of the actual tensor weights or the analytic approximation,
$\hatMSE(Q(W_i), W_i)\propto 4^{-b_i}$, based the analysis in App.~\ref{app:mse-estimation}. The experiments (App.~\ref{app:mse-estimation}) showed a moderate advantage of using the analytical MSE estimates over measured ones, so we adopt the former throughout this study.

\paragraph{Curvature estimation.}
To estimate $\alpha_i$, we use that $Q(W_i,b_i)=W_i$ and $\MSE(Q(W_i,b_i), W_i)=0$ for the layers that are kept in full precision ($b_i= \high$), so effectively Eq.~\eqref{eq:higgs} sums only over those layers that are quantized. 
Then, for layer $i$ we model quantization error by adding Gaussian noise $\delta W_i$ to the weight tensor $W_i$  (with $B_i = \{b_{j \neq i}=\high, b_i=\low\}$) and can directly read off
\begin{equation}
L(W,B_i)-L_{\high} = \alpha_i  \; \hatMSE (Q(W_i,B_i), W_i) = \alpha_i  \; \MSE (W_i + \delta W_i, W_i).
\label{eq:alpha_from_L_and_MSE}
\end{equation}

The dataset on which the loss in Eq.~\ref{eq:alpha_from_L_and_MSE} is evaluated, $\mathcal{D}$, we call 
the \emph{calibration dataset}. It should be large enough for an 
accurate loss approximation while being small enough to 
avoid excessive computational overhead.

In principle, Eq.~\eqref{eq:alpha_from_L_and_MSE} yields $\alpha_1,\dots,\alpha_n$
from $n+1$ forward passes over $\mathcal{D}$, but we note that using several noise perturbations per layer with different variances improves our loss estimation. We used $m \in \{3,\ldots,10\}$ noise levels to estimate $\alpha_i$. Therefore, curvature estimation requires the total of $mn+1$ forward passes. See details on hyperparameter choice in App.~\ref{app:ablation-m-D}.

\algnewcommand{\LineFor}[2]{%
\State \algorithmicfor\ #1\ \algorithmicdo\ #2}

\begin{algorithm}[t]
  \caption{\methodname}
  \label{alg:method}
  \begin{algorithmic}[1]
    \Require Weight tensors $W=(W_1,\dots,W_n)$, calibration data $\mathcal{D}$, update interval $\tau$
    \Repeat 
    \State Evaluate $L_{\high}$ on $\mathcal{D}$
    \For{layers $i = 1, \dots, n$}
        \For{$m$ noise levels $\{\sigma^2_k\}$}
      \State Add noise $\delta W_i \sim N(0, \sigma_k^2)$ to layer weight $W_i$
      \State Determine $\hatMSE(W_i, W_i+\delta W_i)$
      \State Evaluate $L(W + \delta W_i)$ on $\mathcal{D}$
      \EndFor
      \State Estimate sensitivity $\alpha_i$ via least squares using Eq.~\eqref{eq:alpha_from_L_and_MSE}
    \EndFor
    \State $B \gets$ solution of knapsack problem Eq.~\eqref{eq:our-allocation-problem}
    \State Train model with the new bitwidth allocation $B$ for $\tau$ steps
    \Until termination condition
  \end{algorithmic}
\end{algorithm}

\paragraph{Bitwidth assignment.}
Once we have determined sensitivities $\alpha_i$ for all $i=1,\dots,n$, we can solve Eq.~\eqref{eq:our-allocation-problem} with Eq.~\eqref{eq:higgs}
as a proxy for the intractable $L(W, B)$. 
As a discrete selection problem with linear objective and 
linear constraints, this constitutes a \emph{multiple choice knapsack problem (MCKP)}.
For our implementation we use pseudo-polynomial dynamic programming~\citep{pisinger1995minimal}: the budget $C_0$ is discretized into $M$ bins, and each layer cost $d_i b_i$ is rounded up to the whole number of bins. The discrete problem is then solved via a dynamic program that tracks an allocation that gives minimal loss degradation for each cost level. The time complexity of the algorithm is $O(n M |\beta|)$, and in practice, for the models used, we measure $\lesssim 1$ s of knapsack solving time, which is negligible compared to multiple forward passes needed for estimating $\alpha_i$.

\paragraph{\methodname.}
Ideally, the estimates $\alpha_i$ and $\hatMSE$ should stay up-to-date 
at any time during training. 
However, given that they require multiple forward passes to be updated,
it would be unreasonable to change them at every optimization step.
In practice, we observed that sensitivities demonstrate a stable pattern 
during training and infrequent re-allocations suffice. Therefore, 
we propose to update the estimates only at regular intervals of $\tau$ 
steps. 
In our experiments, re-allocation periods lower than 5B training tokens 
did not yield measurable improvement.
Summarizing the discussion above, we obtain the \methodname procedure which 
is our main contribution in this work, see Algorithm~\ref{alg:method} for the
pseudo-code.

\paragraph{Computational and memory overhead.}
\methodname calibration requires only forward passes. To perturb layer $i$ we add $\delta W_i$ to the weights and evaluate the loss on $\mathcal{D}$. The only extra memory is a single noise tensor the size of the largest layer, which is negligible compared to the training state. In contrast, gradient-based methods~\citep{snip_quant} require storing the activations and gradients needed for the backward pass, resulting in a substantially larger memory overhead. 

One recalibration costs $mn+1$ forward passes over $\mathcal{D}$: one unperturbed
pass for $L_{\high}$ and one per (layer, noise level) pair. It is followed by a knapsack solve with
negligible cost. For a batch size $B$ the relative compute overhead is then $\frac{(mn+1) |\mathcal{D}|}{3 \tau B}$, linear in both $m$ and $|\mathcal{D}|$ and inversely proportional to $\tau$. In our experiments the measured wall-clock time overhead was only 5\% for 4B model corresponding to a conservative choice of $\tau=5k$. Time overhead can likely be reduced substantially by increasing $\tau$ (Sec.~\ref{seq:analysis}). More details on hyperparameter choice can be found in App.~\ref{app:ablation-m-D}.

\section{Experimental Results}
\label{sec:results}

\subsection{Experimental Setup}
 We train a decoder-only Transformer~\citep{attention}, with each block following an OLMo-2-style~\citep{olmo2}: multi-head self-attention and a two-layer MLP, with RMSNorm~\citep{rmsnorm} applied to each sub-layer output before the residual. Attention uses rotary position embeddings~\citep{rope}, and Qwen3-style~\citep{qwen3} QK-RMSNorm. We use model sizes of 350M, 500M, 800M, and 4B.
We start our experiments with a 350M model on 18B tokens and evaluate all the strategies, then select the best-performing ones for larger up to 4B model runs. 

The training data is ClimbMix dataset~\citep{climbmix}. We use a subset pre-tokenized with TokenMonster (vocab size 32 000). We used AdamW~\citep{kingma2017adammethodstochasticoptimization, loshchilov2019decoupledweightdecayregularization} and Muon~\citep{muon} optimizers. The learning rate (LR) schedule used is WSD~\citep{wsd}. The optimal LR is found via a grid search. Context length range is from 1024 for 350M model to 4096 for 4B model, more details are provided in App.~\ref{app:training-setup}.

\paragraph{Quantization-aware training.} 
Linear layers in the blocks (Q, K, V, O, and MLP projections) are the only quantized modules; embeddings and the LM head remain in high precision. We follow the e2e QAT setup by \cite{panferov2026quartet}, which performs all of the weights-activations matmul and gradients-weight/activations matmuls in FP4. In our experiments, the candidate low precision formats for matmix multiplications are NVFP4, MXFP6, MXFP8. For mixed-precision runs, a global bitwidth budget is defined as the total memory cost of the quantized weights divided by the number of quantized parameters, i.e., bits per parameter. Unless stated otherwise, all matmuls for a layer, forward and backward, are performed in the same precision. The weight master copy and optimizer state are stored in FP32. 

\paragraph{Baselines.} 
We compare \methodname against uniform NVFP4, MXFP8, the \emph{last-layers} strategy of~\citet{nvidia2026pretraininglargelanguagemodels}, and SNIP~\citep{snip_quant}. \lastlayers keeps the last layers of the model at \high precision, with the cutoff chosen to satisfy the budget. SNIP solves a similar knapsack problem, but estimates layer costs using gradient-based forward sensitivity and backward-update divergence. We also consider a version of SNIP, named SNIP weight-gradient-only, as it performs better than the original SNIP recipe in pretraining. See App.~\ref{app:snipwgo} for details.

\subsection{Forward Error Is Sufficient for Loss Estimation}
Before we report our method's performance, we provide a justification for our 
choice of modeling the quantization error purely as a function of the $\MSE$, \ie weight quantization error, Eq.~\eqref{eq:MSE}. Other mixed-precision training methods model loss degradation using the error introduced during the backward pass~\citep{snip_quant}. We test whether the signal from the backward pass is necessary for predicting the loss under QAT.

For this, we compare two training settings for a 500M model trained on 74B tokens, with various precision allocation methods. In the first setting, the layer-wise precision
assignment is applied to both the forward and backward matrix multiplications. In the second, the precision assignment is applied only to the forward matmuls, while all the backward matmuls are computed in FP4 (FP4 bwd). The second setup introduces additional backward pass quantization error that is not captured by our forward error metric (Eq.~\ref{eq:higgs}).
\begin{wrapfigure}[12]{r}{0.5\textwidth}
\centering
\includegraphics[width=0.5\textwidth]{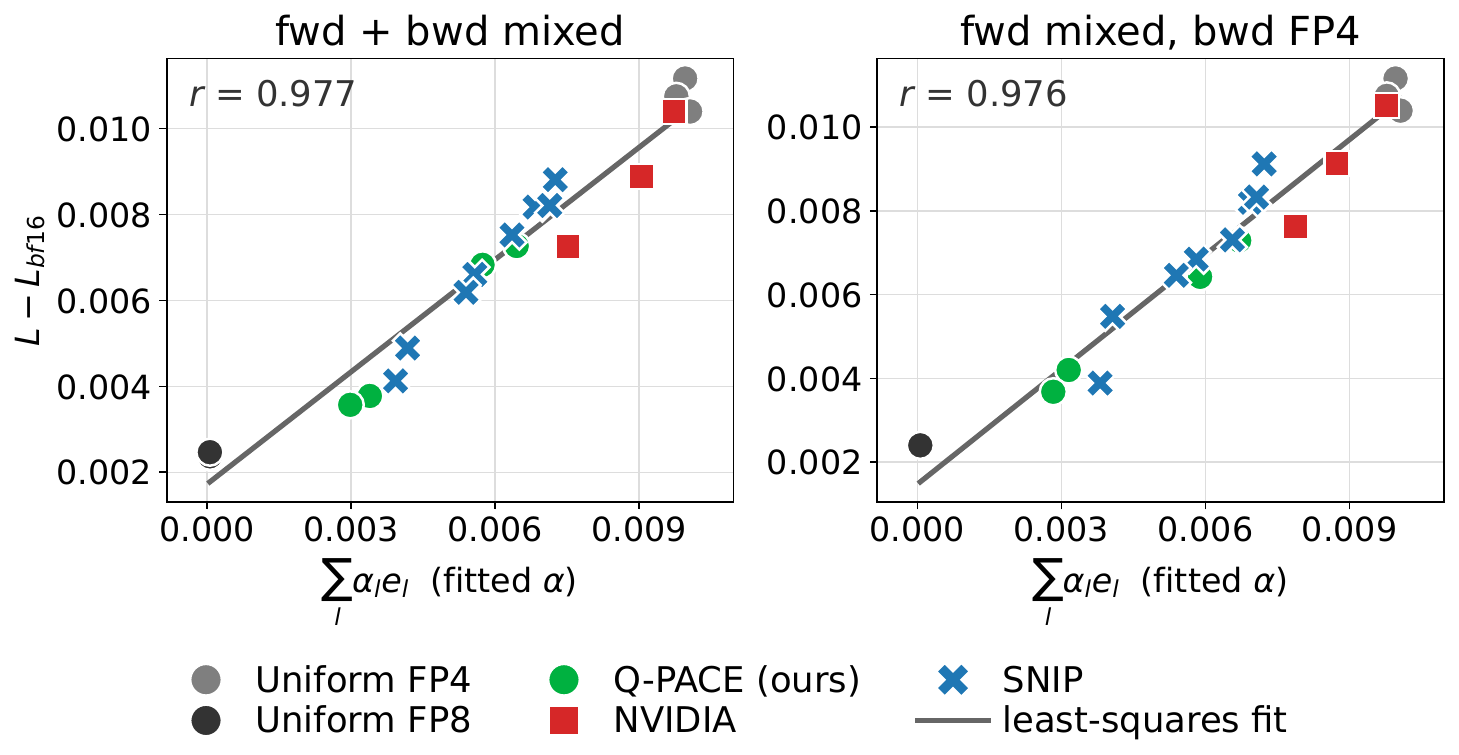}
\caption{\small Test loss vs. the modeled sensitivity-weighted forward error for 500M models' final allocations.}
\label{fig:nmse-vs-loss}
\end{wrapfigure}

Figure~\ref{fig:nmse-vs-loss} shows the final test loss as a function of our loss estimate (Eq.~\ref{eq:higgs}), which uses forward error only. In both settings, the final loss remains linear in the forward error. We calculate the Pearson correlation  between the loss estimate and final test loss, $r$. When forward and backward precisions are the same per layer, the correlations are $r=0.98$ for our model. Keeping the backward pass uniformly in FP4 does not change this, and the corresponding Pearson correlation remains $r=0.98$.

We conclude that, although additional perturbation is indeed introduced during the backward pass in QAT, backward pass metrics are \textit{not required} to predict the impact on the loss. Across the precision allocation methods tested, forward error alone can be a good predictor of final loss. Supporting results in App.~\ref{app:losslinear}.

\subsection{Pretraining Experiments}

\begin{table}[t]
\centering
\caption{ \small
Final test loss across models.
Budget equals target average bits per parameter.
``Rec.'' (recovery) is the share of the FP4-to-FP8 gap recovered, $(L_{\text{FP4}}-L)/(L_{\text{FP4}}-L_{\text{FP8}})$.
The column ``gain'' reports the ratio between the gain of \methodname over FP4
and the gain of the NVIDIA allocation over FP4,
$(L_{\text{FP4}}-L_{\text{\methodname}})/(L_{\text{FP4}}-L_{\text{NVIDIA}})$. The std error measured over 3 random seeds for 350M model is $\leq 0.001$. Best loss per row is shown in bold. 
}
\label{tab:pretraining-results}

\resizebox{\textwidth}{!}{%
\begin{tabular}{ll|c|c|cc|cc|cc|cc}
\toprule
Model
& Budget
& FP4
& FP8
& NVIDIA
& Rec.
& SNIP
& Rec.
& \methodname
& Rec.
& gain \\
\midrule

\multirow{4}{*}{$350$M, $16$B tok}
& $5.0\ \{4,6,8,16\}$
& \multirow{4}{*}{$2.886$}
& \multirow{4}{*}{$2.848$}
& $2.884$ & $6\%$
& $2.872$ & $37\%$
& $\mathbf{2.868}$ & $\mathbf{49\%}$
& $7.5\times$ \\

& $5.5\ \{4,6,8,16\}$
& 
& 
& $2.880$ & $15\%$
& $2.867$ & $51\%$
& $\mathbf{2.858}$ & $\mathbf{74\%}$
& $4.8\times$ \\

& $6.5\ \{4,8,16\}$
& 
& 
& $2.876$ & $28\%$
& $2.864$ & $59\%$
& $\mathbf{2.857}$ & $\mathbf{75\%}$
& $2.7\times$ \\

& $6.5\ \{4,16\}$
& 
& 
& $2.876$ & $28\%$
& $2.874$ & $33\%$
& $\mathbf{2.870}$ & $\mathbf{43\%}$
& $1.6\times$ \\

\midrule

\multirow{4}{*}{$500$M, $74$B tok}
& $5.0\ \{4,6,8,16\}$
& \multirow{4}{*}{$2.693$}
& \multirow{4}{*}{$2.657$}
& $2.692$ & $2\%$
& $2.678$ & $40\%$
& $\mathbf{2.673}$ & $\mathbf{55\%}$
& $32.4\times$ \\

& $5.5\ \{4,6,8,16\}$
& 
& 
& $2.688$ & $14\%$
& $2.673$ & $54\%$
& $\mathbf{2.664}$ & $\mathbf{80\%}$
& $5.8\times$ \\

& $6.5\ \{4,8,16\}$
& 
& 
& $2.683$ & $26\%$
& $2.672$ & $58\%$
& $\mathbf{2.663}$ & $\mathbf{81\%}$
& $3.1\times$ \\

& $6.5\ \{4,16\}$
& 
& 
& $2.683$ & $26\%$
& $2.680$ & $36\%$
& $\mathbf{2.675}$ & $\mathbf{48\%}$
& $1.8\times$ \\

\midrule

\multirow{2}{*}{$800$M, $80$B tok}
& $5.0\ \{4,16\}$
& \multirow{3}{*}{$2.616$}
& \multirow{3}{*}{$2.581$}
& $2.613$ & $9 \%$
& $2.613$ & $9 \%$
& $\mathbf{2.609}$ & $\mathbf{20\%}$
& $2.2\times$ \\

& $5.5\ \{4,16\}$
& 
& 
& $2.611$ & $14 \%$
& $2.612$ & $ 11 \%$
& $\mathbf{2.607}$ & $\mathbf{26 \%}$
& $1.9\times$ \\

& $6.0$ $\{4,16\}$
& 
&
& $2.607$ & $26\%$
& $2.605$ & $31\%$
& $\mathbf{2.602}$ & $\mathbf{40\%}$
& $1.5 \times$\\



\midrule

$4$B, $105$B tok
& $6.5\ \{4,16\}$
& $2.278$
& $2.251$
& $2.270$ & $31\%$
& $2.270$ & $32\%$
& $\mathbf{2.266}$ & $\mathbf{47\%}$
& $1.5\times$ \\ 

\bottomrule
\end{tabular}%
}
\end{table}

\begin{figure}[t]
\sbox0{\footnotesize\setlength{\tabcolsep}{4pt}%
\begin{tabular}{@{}lccc@{}}
\toprule
 & Loss & Rec. & Acc. \\
\midrule
uniform FP4 & $0.7375_{\pm0.0005}$ & $0\%$ & $55.0_{\pm0.3}$ \\
uniform FP8 & $0.6790_{\pm0.0001}$ & $100\%$ & $57.3_{\pm0.2}$ \\
\midrule
\multicolumn{4}{@{}l}{\textit{budget 5 $\{4,6,8,16\}$}} \\
NVIDIA & $0.7283_{\pm0.0004}$ & $16\%$ & $55.0_{\pm0.2}$ \\
SNIP & $0.7115_{\pm0.0003}$ & $44\%$ & $\mathbf{56.2}_{\pm0.2}$ \\
\methodname & $\mathbf{0.7077}_{\pm0.0003}$ & $\mathbf{51\%}$ & $55.9_{\pm0.2}$ \\
\midrule
\multicolumn{4}{@{}l}{\textit{budget 5.5 $\{4,6,8,16\}$}} \\
NVIDIA & $0.7221_{\pm0.0003}$ & $26\%$ & $55.0_{\pm0.3}$ \\
SNIP & $0.7004_{\pm0.0004}$ & $63\%$ & $56.4_{\pm0.1}$ \\
\methodname & $\mathbf{0.6961}_{\pm0.0005}$ & $\mathbf{71\%}$ & $\mathbf{56.5}_{\pm0.3}$ \\
\midrule
\multicolumn{4}{@{}l}{\textit{budget 6.5 $\{4,8,16\}$}} \\
NVIDIA & $0.7174_{\pm0.0005}$ & $34\%$ & $54.8_{\pm0.2}$ \\
SNIP & $0.6973_{\pm0.0003}$ & $69\%$ & $\mathbf{56.7}_{\pm0.1}$ \\
\methodname & $\mathbf{0.6919}_{\pm0.0006}$ & $\mathbf{78\%}$ & $56.5_{\pm0.2}$ \\
\midrule
\multicolumn{4}{@{}l}{\textit{budget 6.5 $\{4,16\}$}} \\
NVIDIA & $0.7174_{\pm0.0005}$ & $34\%$ & $54.8_{\pm0.2}$ \\
SNIP & $0.7099_{\pm0.0011}$ & $47\%$ & $\mathbf{56.0}_{\pm0.3}$ \\
\methodname & $\mathbf{0.7085}_{\pm0.0005}$ & $\mathbf{50\%}$ & $55.8_{\pm0.3}$ \\
\bottomrule
\end{tabular}}%
\begin{minipage}[t]{\dimexpr\textwidth-\wd0-1em\relax}\vspace{0pt}\centering
\includegraphics[width=\textwidth]{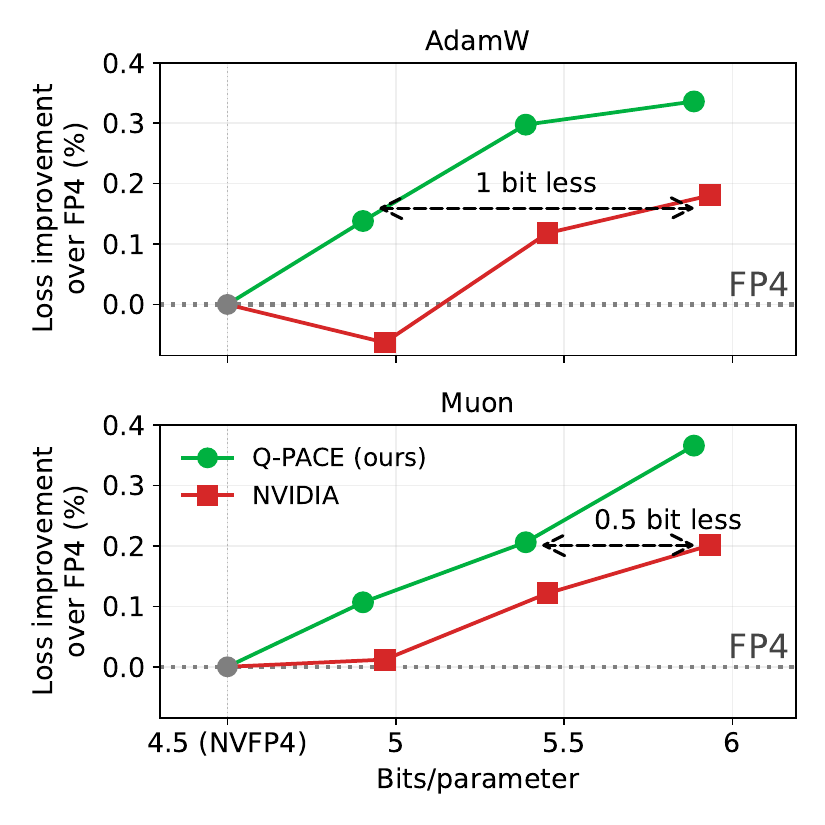}
\caption{Loss improvement over uniform FP4 for 800M models with varying budgets. To achieve the same quality models trained with \methodname need 1 bit/parameter less than those trained with \lastlayers when using AdamW and 0.5 bit/parameter less when using Muon.}
\label{fig:pareto-800m-adam-muon}
\end{minipage}\hfill
\begin{minipage}[t]{\wd0}\vspace{0pt}
\raisebox{\dp0}{\usebox0}
\captionof{table}{SFT held-out loss, recovered share of the FP4-to-FP8 gap (Rec.), and mean downstream accuracy of Qwen3-4B-Base on Tulu-3 (mean and standard deviation over three seeds).}
\label{tab:sftgrid}
\end{minipage}
\end{figure}

We compare the mixed-precision training methods by their final test loss in Table~\ref{tab:pretraining-results}. Across all model scales, precision budgets, and candidate precision sets, \methodname achieves lower final test loss than both the \lastlayers heuristic and SNIP~\citep{snip_quant}. These gains are obtained under the same average precision budget. \methodname recovers substantially more of the FP4-to-FP8 loss gap than other methods, up to $81\%$ of the gap. \methodname provides $1.5\times$--$32\times$ the loss improvement of the \lastlayers across the reported settings. 
This shows that allocating precision according to curvature-aware layer sensitivity is consistently more effective compared to other baselines.

The benefit of sensitivity-aware allocation is particularly clear when more than 2 precision levels are available. At the same $6.5$-bit budget, replacing the bitwidth set $\{4,16\}$ with $\{4,8,16\}$ increases the recovered gap of \methodname from $43\%$ to $75\%$ for the 350M model and from $48\%$ to $81\%$ for the 500M model. Figure~\ref{fig:dyn4b} further shows how our method improves with additional \textit{bitwidth choices}. With the same budget of 6.5 bits/parameter, using $\{4, 8,16\}$ as available formats allows to close $79\%$ of gap to the FP8 level, whereas with only $\{4,16\}$ choices it reaches $47\%$. 

\subsection{Supervised Fine-Tuning Experiments}
\label{subsec:sft}
We extend our experiments to supervised fine-tuning (SFT). We fine-tune Qwen3-4B-Base~\citep{qwen3technicalreport} on Tulu-3~\citep{tulu3} with budgets from 5.0 to 6.5 over three random seeds. Table~\ref{tab:sftgrid} reports the held-out loss, the recovered share of the FP4-to-FP8 gap, and the mean accuracy over the downstream benchmarks. \methodname reaches the lowest loss at every budget, with a substantial margin over the baselines. \methodname outperforms the NVIDIA baseline in the downstream tasks average. However, it falls very close to SNIP and within the noise range. Detailed downstream benchmarks are provided in App.~\ref{app:sft}.

\subsection{Performance-memory Trade-off}

Next, we study the trade-off between model quality and precision budget. 
Figure~\ref{fig:pareto-800m-adam-muon} shows test loss improvement over the FP4 baseline 800M model across bitwidths. 
For both AdamW and Muon, \methodname forms an empirical Pareto frontier: at a fixed precision budget, it achieves a larger loss improvement than the \lastlayers, while at a matched loss it requires fewer bits per weight. 
For Adam, our method outperforms \lastlayers strategy while using 1 bit per parameter less, that is 16\% less model size in memory.
With Muon, the corresponding saving is approximately 0.5 bits per weight. 
Altogether, this indicates that the benefit of curvature-aware sensitivity is not restricted to a single precision budget or optimizer.

Since a model can be deployed under various weight and activation formats during inference, we additionally measure the loss--memory trade-off for different \textit{inference formats}. 
Figure~\ref{fig:inference-pareto} shows the evaluation of 800M model checkpoints under several inference configurations, including $W_{\mathrm{mix}}A_{\mathrm{mix}}$, $W_{\mathrm{mix}}A_8$, and $W4A8$. Here $_\text{mix}$ denotes mixed precision across layers.

Across the evaluated configurations, models trained with \methodname again lie on the empirical Pareto frontier: for a given inference-memory budget they achieve lower loss, or equivalently reach a comparable loss with less memory. 
Thus, the advantage of sensitivity-guided precision allocation during training is retained under variable inference formats.

\begin{figure}[t]
    \centering
    \includegraphics[width=0.7\linewidth]{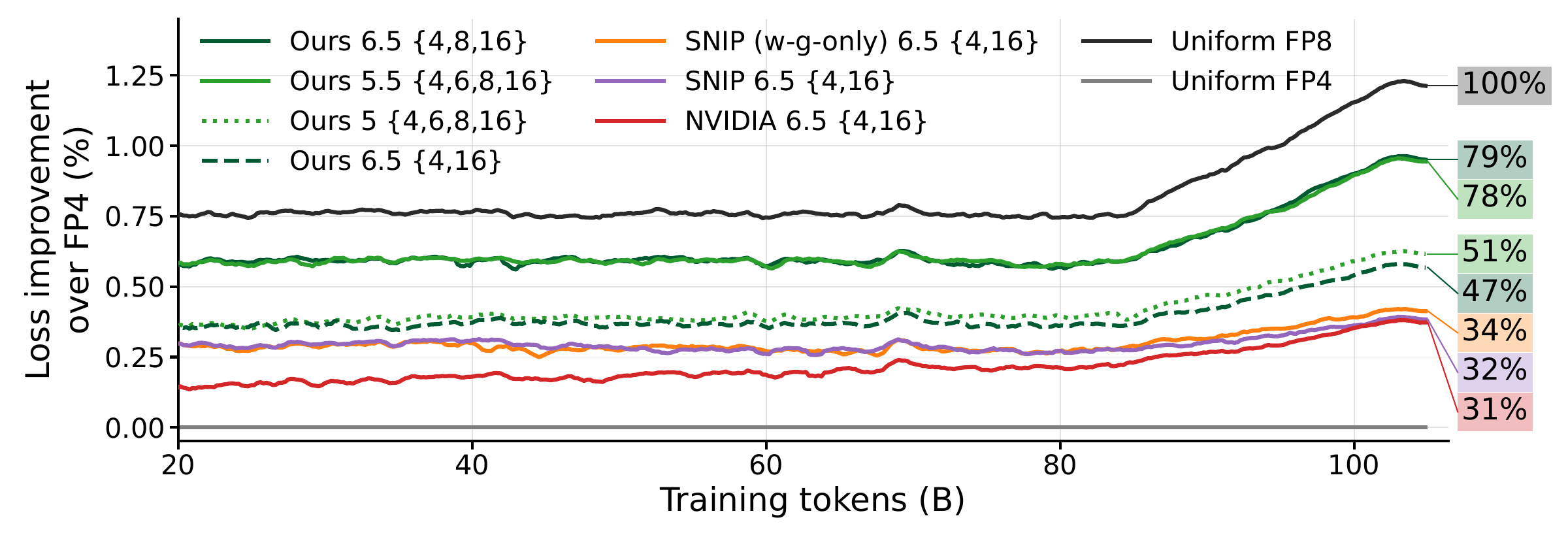}
    \caption{\small Pretraining experiment for a 4B model. Test-loss improvement over uniform FP4; labels give the recovered gap between FP4 and FP8 trainings, realized bits/parameter, and the precision options in parentheses.}
    \label{fig:dyn4b}
\end{figure}

\begin{figure}[t]
    \centering
    \begin{subfigure}[]{0.49\textwidth}
        \centering
    \includegraphics[width=0.9\linewidth]{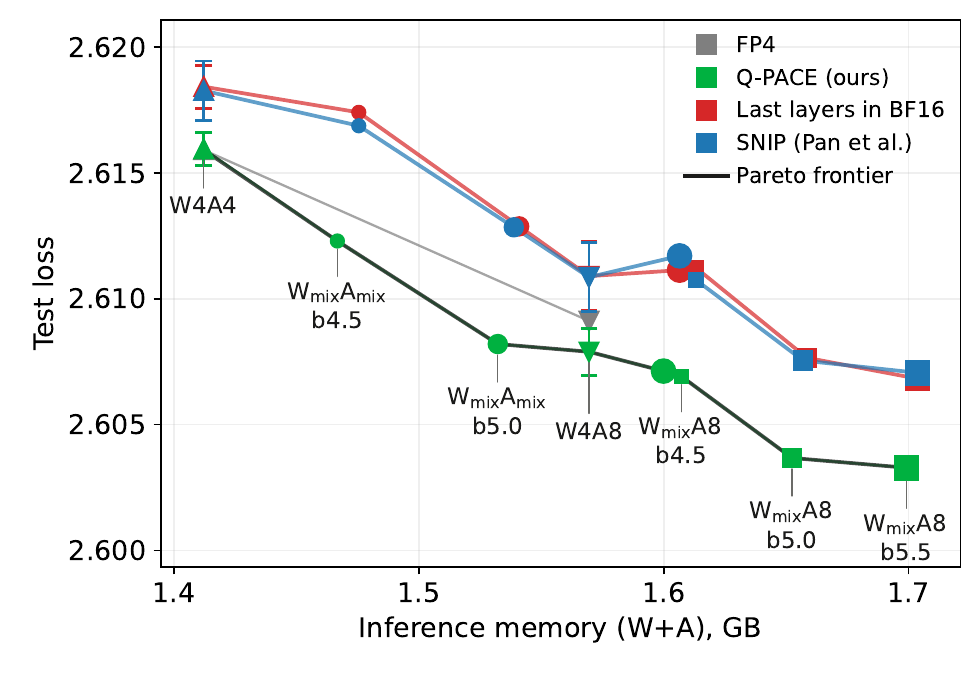}
    \caption{Loss -- inference memory tradeoff for 800M model across different inference formats: $W_{\text{mix}} A_{\text{mix}}$, $W_{\text{mix}} A8$, $W4A8$.}
    \label{fig:inference-pareto}
    \end{subfigure}
    \hfill
    \begin{subfigure}[]{0.49\textwidth}
        \centering
        \includegraphics[width=0.9\linewidth]{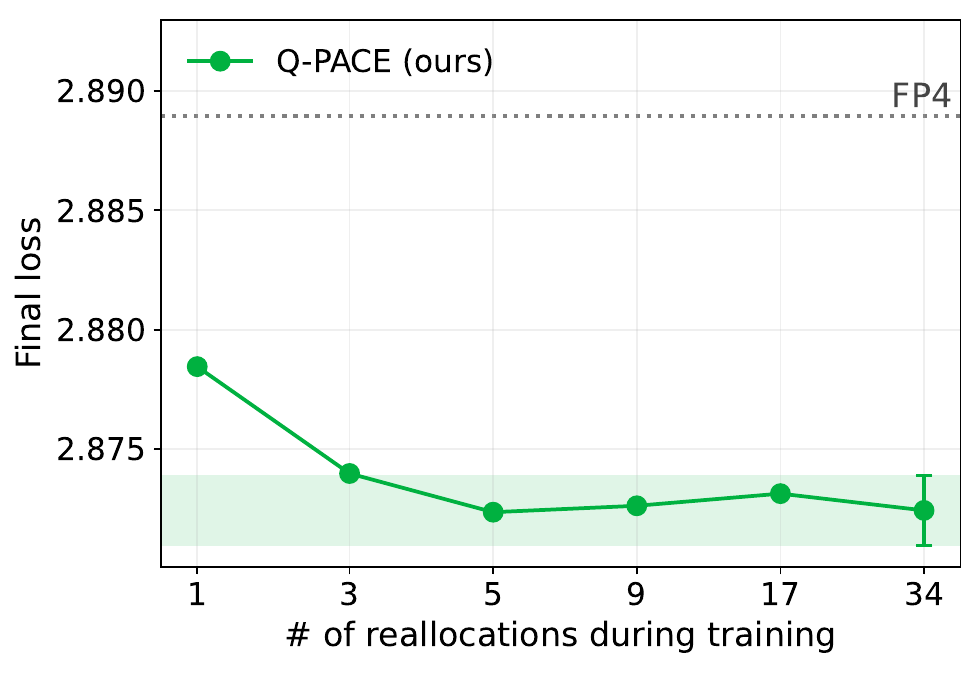}
        \caption{Final loss as a function of number of reallocations during training for the 350M model. }
        \label{fig:frequency-ablation}
    \end{subfigure}

\end{figure}

\section{Analysis}
\label{seq:analysis}
Next we analyze the properties and evolution of bitwidth allocations produced by our method. 
\paragraph{Sensitivity is structured by layer type and depth. } Figure~\ref{fig:4b-allocation} shows the final precision allocation produced by \methodname and SNIP for the 4B model. The allocation is highly non-uniform: the assigned precision depends both on the layer type and its position in the model. 
The depth dependence is particularly pronounced in the final block, where \methodname promotes five of the six quantized linear layers to the full precision.
This observation helps to explain why the \lastlayers heuristic is effective relative to uniform FP4 training: later layers are indeed highly sensitive to quantization.

However, sensitivity is not solely a function of depth. Layers of different types within the one block receive different precisions, while sensitive layers can also occur substantially earlier in the network.
In particular, K and V projection layers receive a substantial fraction of the higher-precision budget under both \methodname and SNIP. At the same time, SNIP does not promote MLP projections at all, while \methodname assigns high precision to MLP projections toward the end of the model.

\paragraph{Sensitivity evolves during training.} Figure~\ref{fig:4b-evolution} shows the fraction of each layer type assigned to BF16 throughout training.
While the number of layers of each type, selected by SNIP, remains nearly constant over time, the allocation produced by \methodname evolves. The largest changes occur between the early and middle stages of training, after which the allocation becomes substantially more stable. K and V projections remain repeatedly promoted to BF16, while the share of promoted Q projections decreases over time, and the share of MLP grows.

Therefore, the sensitivity pattern has both a constant factor, determined by layer type and block number, as well as a time-dependent factor. The former explains why simple static heuristics can provide some benefit, whereas the latter motivates periodic reallocation during training. Additional analysis results can be found
in App.~\ref{app:sensitivity-analysis}

\paragraph{How frequently is reallocation necessary?}
Next, we analyze the impact of the reallocation period $\tau$. Figure~\ref{fig:frequency-ablation} reports final loss while varying the number of reallocations over training. For a 350M model, reallocating more than five times over the full training run provides no additional benefits. 

This result demonstrates that the gain of \methodname is not only due to finding a better \emph{initial} mixed-precision allocation. Adapting the allocation dynamically during training provides an additional benefit. 
Secondly, it shows that sensitivity changes on a much slower timescale than optimization steps, allowing to keep the computational overhead of \methodname recalibration marginal. 
\begin{figure}[t]
    \centering
    \begin{subfigure}[]{0.49\textwidth}
        \centering
    \includegraphics[width=\linewidth]{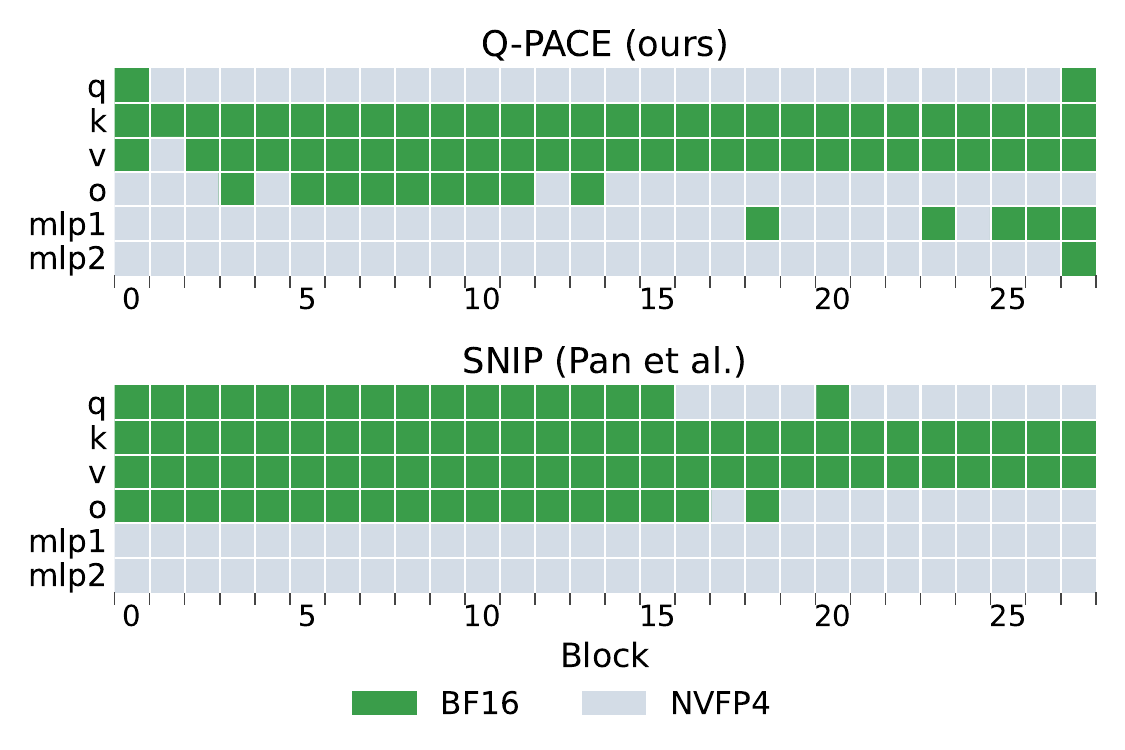}
    \caption{Final format allocation for 4B model at 6.5 bit/param budget for \methodname and SNIP~\citep{snip_quant}. Both method favor K and V projection layers, with \methodname also promoting the MLP in the last blocks.}
    \label{fig:4b-allocation}
    \end{subfigure}
    \hfill
    \begin{subfigure}[]{0.49\textwidth}
        \centering
        \includegraphics[width=\linewidth]{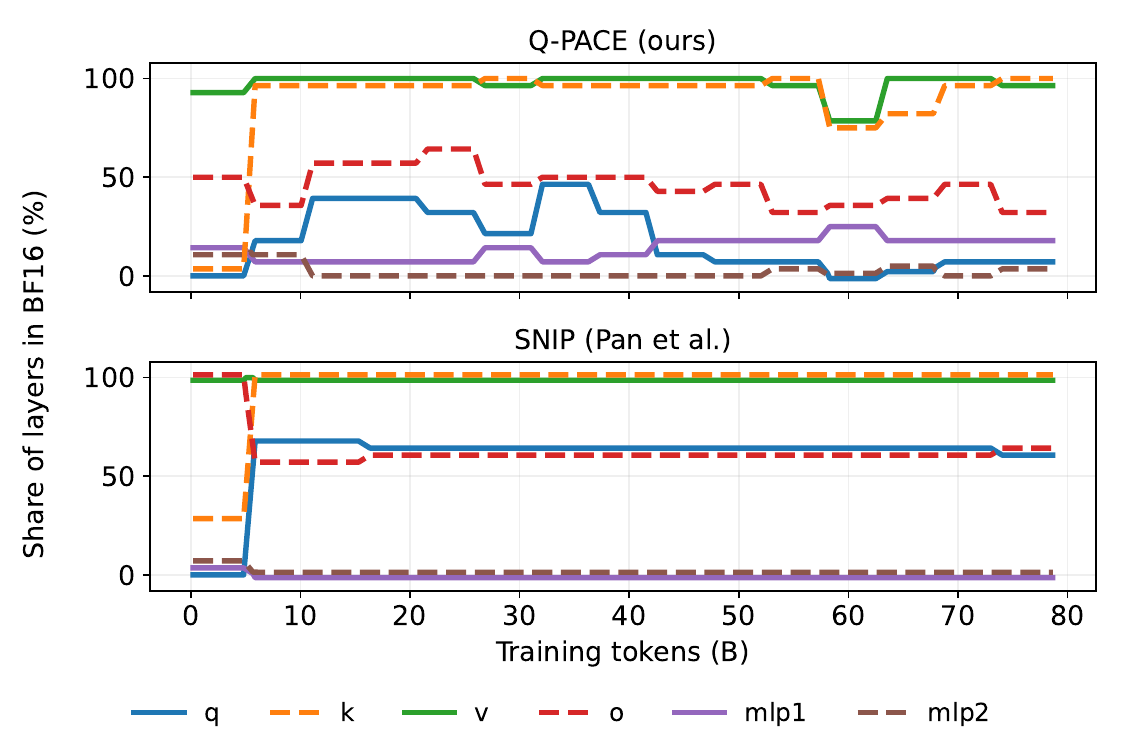}
        \caption{Share of layers promoted to BF16 in 4B model at 6.5 bit/param, by layer type, over training duration.}
        \label{fig:4b-evolution}
    \end{subfigure}

\end{figure}

\section{Related Work}
\label{sec:related}
Post-training quantization (PTQ)~\citep{dettmers2022llmint8,xiao2023smoothquant, frantar2023gptq} has become a standard approach for reducing the memory and inference cost of LLMs.
Most PTQ methods apply the same bitwidth to every layer.

\paragraph{Non-uniform compression.}
A parallel line of work distributes compression non-uniformly across layers.
AMC~\citep{he2018amc} and HAQ~\citep{wang2019haq} used reinforcement learning with hardware feedback for this, while DNAS~\citep{wu2018dnas} relaxed the search into a differentiable one. 
HAWQ~\citep{hawc,dong2020hawqv2,yao2021hawqv3} instead scores layers by second-order information, and selects bitwidths by integer programming.
SPDY~\citep{frantar2022spdy} and OBC~\citep{frantar2022obc} cast the allocation as a dynamic program over per-layer compression levels.
EvoPress~\citep{sieberling2025evopress} showed that such additive (``error-monotone'') models can fail, and replaced them with an evolutionary search.
For LLMs, mixed-precision schemes range from keeping salient weights or groups in higher precision~\citep{dettmers2024spqr,kim2024squeezellm,huang2025slimllm}, to layer-wise allocation driven by per-layer sensitivity estimates and knapsack solvers~\citep{malinovskii2025higgs,helcig2026statisticallylosslessquantizationlargelanguage}, or even more complex manifold optimization~\citep{helcig2026rco}.
Other work accounts for interactions between layers~\citep{schaefer2023interlayer} or for perturbations~\citep{hu2025pqi}.
The closest work to ours is HIGGS~\citep{malinovskii2025higgs}, whose linearity theorem underlies our sensitivity model.
\methodname differs from it in three main ways: 
1. HIGGS allocates bitwidths once; we allocate \emph{during} training;
2. HIGGS compresses only the weights, whereas we compress weights, activations, and gradient multiplications.
3., We investigate hardware-supported microscaling formats (NVFP4, MXFP6, MXFP8, BF16) and use an analytic $4^{-b}$ distortion model, so no per-layer, per-format errors are needed.

\paragraph{Mixed precision during training.}
Non-uniform precision \emph{during} training has received less attention.
Low-precision training recipes typically keep a hand-picked set of components in higher precision, e.g.  DeepSeek-V3 keeps the embeddings, output head, MoE gating, normalization, and attention in higher precision~\citep{deepseekv3}, and NVFP4 pretraining keeps the last blocks in BF16~\citep{nvidia2026pretraininglargelanguagemodels}.
Earlier work on convolutional networks learned per-layer bitwidths jointly with the weights during QAT~\citep{wu2018dnas,uhlich2020mixed,vanbaalen2020bayesianbits}.
However, these methods target the precision of the final inference model whereas we periodically re-allocate precision during training under a fixed global budget.
The work closest to ours is SNIP~\citep{snip_quant}, which periodically re-assigns per-layer precision during LLM pretraining by solving a knapsack problem.
Its scores combine a first-order Taylor estimate of the forward loss change with an optimizer-dependent estimate of the divergence in the weight updates.
\methodname replaces these with a second-order, loss-based sensitivity model, estimated from forward passes alone.
We show that this yields substantially better accuracy.

\section{Conclusion}
We propose \methodname, a new method for adaptive mixed-precision LLM training, which probes linear layers and estimates their sensitivity. This information is then used for periodic reassignment of precision across layers during training. With no memory overhead, our method consistently improves over existing baselines and allows for substantial memory savings at a matched performance. Across pretraining experiments from 350M to 4B parameters and supervised fine-tuning of Qwen3-4B, it outperforms other methods and recovers up to $81\%$ of the NVFP4-to-MXFP8 loss gap at 6.5 bit/parameter. 

Our study is limited to a restricted range of model scales and precision formats, and primarily considers settings where the same layer-wise precision is used in the forward and backward passes. Extending dynamic allocation to larger models, broader format choices, and independently optimized forward and backward precisions are promising directions for future work.

\subsection*{Acknowledgments}

We would like to thank our contacts at Verda, Paul Chang and Antonio Dominguez, for hardware support that was essential to this project. 

The computational results have been achieved using the Austrian Scientific Computing (ASC) infrastructure. This research was supported by the Scientific Service Units (SSU) of IST Austria through resources provided by Scientific Computing (SciComp). 

 We thank NVIDIA and Google corporation for their grants, which supported part of this research. This research was supported by the ISTA Responsible AI Program, made possible through the generous support of Garrett Camp and the Camp Foundation.

\subsection*{AI use statement}

We used ChatGPT and Claude to assist in organizing and editing the manuscript, as well as verifying mathematical reasoning and our interpretation of the results. Additionally, we employed Claude Code to visualize our experimental results. We have thoroughly reviewed and verified the final manuscript, including all citations, experimental details, and AI-assisted edits, and we take full responsibility for its content.

\subsection*{Reproducibility Statement}

We provide detailed information intended to support reproduction of our results. The model architectures, training configurations, datasets, optimizer settings, learning-rate schedules, batch sizes, context lengths, and recalibration frequencies are described in Appendix~\ref{app:training-setup}. The numerical formats and implementation details for NVFP4, MXFP6, MXFP8, and BF16 training, including the hardware used, are specified in Appendix~\ref{app:training-setup}. The Q-PACE procedure is given explicitly in Algorithm 1, including the sensitivity-estimation and precision-allocation steps, and we report the calibration and reallocation hyperparameters used in our experiments. Appendix~\ref{app:ablation-m-D} studies the sensitivity of the method to its main calibration hyperparameters, while the remaining appendices provide additional  ablations, analyses, and downstream results. All reported experimental results are obtained from the configurations described in the paper.

\bibliographystyle{iclr2027_conference}
\bibliography{iclr2027_conference}

\clearpage
\appendix

\section{Modeling loss under quantization noise}
\label{app:loss-approximation}
\begin{lemma}[Expected loss under zero-mean quantization noise]
\label{lem:loss-approximation}

Let $i=1, \dots , n$ index model layers and $W=(W_1,\ldots,W_n)$ denote the set of full-precision model parameters. 
Let $w_i=\operatorname{vec}(W_i)\in\mathbb{R}^{d_i}$ be vectorized form of $W_i$.
Assume that in a neighborhood
of $W$, $L(W)$ is twice continuously differentiable and that its Hessian is Lipschitz.

Let $\delta w_i:= \operatorname{vec}(\delta W_i)$ be random perturbations with finite second and third moments satisfying 
\begin{equation}
       \mathbb{E}[\delta w_i]=0,
    \qquad
    \mathbb{E}[\delta w_i\delta w_i^\top]=\Sigma_i ,
    \label{eq:perturbation-covariance}
\end{equation}
and assume that perturbations of different layers are uncorrelated,
\begin{equation}
    \mathbb{E}[\delta w_i\delta w_j^\top]=0,
    \qquad i\neq j.
    \label{eq:cross-layer-uncorrelated}
\end{equation}
Then, for sufficiently small perturbations,
\begin{equation}
\label{eq:general-loss-approximation}
    \mathbb{E}\!\left[L(W+\delta W)-L(W)\right]
    =
    \frac{1}{2}
    \sum_{i=1}^n
    \Tr\!\left(H_{ii}\Sigma_i\right)
    +
    O\!\left(\mathbb{E}\|\delta W\|^3\right),
\end{equation}
where $\displaystyle  H_{ij} :=    \frac{\partial^2 L(W)}{\partial w_i\,\partial w_j} \in \mathbb{R}^{d_i\times d_j}$ denotes the $(i,j)$ block of the Hessian.
\end{lemma}
\begin{proof}
A second-order Taylor expansion of $L$ around $W$ gives
\begin{equation}
\begin{aligned}
    L(W+\delta W)-L(W)
    &=
    \sum_i
    \left\langle \nabla_{w_i}L(W),\delta w_i\right\rangle +
    \frac{1}{2}
    \sum_{i,j}
    \delta w_i^\top H_{ij}\delta w_j
    +
    R_3(\delta W).
\end{aligned}
\label{eq:taylor-general}
\end{equation}
Since the Hessian is locally Lipschitz, there exists
a constant $C>0$, independent of $\delta W$, such that
$    |R_3(\delta W)|
    \leq C\|\delta W\|^3.
$
Consequently,
\begin{equation}
    \left|\mathbb{E}[R_3(\delta W)]\right|
    \leq
    C\,\mathbb{E}\|\delta W\|^3.
\end{equation}
Taking expectations in Eq.~\eqref{eq:taylor-general}, the first-order
term vanishes since $\mathbb{E}[\delta w_i]=0$. 
By Eq.~\eqref{eq:cross-layer-uncorrelated} second order term gives
\begin{equation}
    \mathbb{E}\!\left[
        \delta w_i^\top H_{ij}\delta w_j
    \right]
    = \Tr\!\left(
        H_{ij}
        \mathbb{E}[\delta w_j\delta w_i^\top]
    \right) = \begin{cases}
    \Tr( H_{ii} \Sigma_i ) \text{ for } i = j,\\ 
 0  \text{ for }  i \neq j.
\end{cases}
\end{equation}
Substituting these identities into Eq.~\eqref{eq:taylor-general}
gives Eq.~\eqref{eq:general-loss-approximation}.
\end{proof}

Next we consider two specific noise models.
\paragraph{Case 1: Isotropic quantization noise.}
Suppose that the quantization perturbation of layer $i$ is modeled as
isotropic zero-mean noise,
\begin{equation}
    \mathbb{E}[\delta w_i]=0,
    \qquad
    \mathbb{E}[\delta w_i\delta w_i^\top]
    = \Sigma_i =
  \sigma_i^2   \operatorname{I} .
\end{equation}
Then Lemma~\ref{lem:loss-approximation} gives
\begin{equation}
\label{eq:isotropic-loss-approximation}
    \mathbb{E}\!\left[L(W+\delta W)-L(W)\right]
    =
    \frac{1}{2}
    \sum_i
    \Tr(H_{ii})\,\sigma_i^2
    +
    O\!\left(\mathbb{E}\|\delta W\|^3\right).
\end{equation}
By definition,
\begin{equation}
\MSE(W_i+\delta W_i,W_i)
    :=
    \frac{1}{d_i}
    \mathbb{E}\|\delta w_i\|_2^2
    =
    \frac{1}{d_i}\Tr(\sigma_i^2 I)
    =
    \sigma_i^2.
\end{equation}
Therefore, defining $  \displaystyle \alpha_i^{\mathrm{iso}}
    :=    \frac{1}{2}\Tr(H_{ii})$
we obtain
\begin{equation}
\label{eq:isotropic-mse-loss}
    \mathbb{E}\!\left[L(W+\delta W)-L(W)\right]
    =
    \sum_i
    \alpha_i^{\mathrm{iso}}\,
    \operatorname{MSE}(W_i+\delta W_i,W_i)
    +
    O\!\left(\mathbb{E}\|\delta W\|^3\right).
\end{equation}

\paragraph{Case 2: Floating-point quantization noise.}
For floating-point formats, the magnitude of the quantization error
is proportional to the magnitude of the weight.
We therefore consider the multiplicative perturbation model
\begin{equation}
\label{eq:multiplicative-noise}
    \delta w_i
    =
    \operatorname{diag}(w_i)\nu_i
    =
    w_i\odot\nu_i,
\end{equation}
where the entries of $\nu_i$ are i.i.d. with
$\mathbb{E}[\nu_i]=0,
    \quad
    \mathbb{E}[\nu_i\nu_i^\top]
    =
    \sigma_i^2 I $, similarly to \cite{chen2026wush}.

The covariance of the resulting perturbation is
\begin{align}
    \Sigma_i
    &=
    \mathbb{E}[\delta w_i\delta w_i^\top]
=    \operatorname{diag}(w_i)
    \mathbb{E}[\nu_i\nu_i^\top]
    \operatorname{diag}(w_i)
=    \sigma_i^2\operatorname{diag}(w_i^2).
\end{align}
Lemma~\ref{lem:loss-approximation} yields
\begin{equation}
\label{eq:multiplicative-loss-approximation}
    \mathbb{E}\!\left[L(W+\delta W)-L(W)\right]
    =
    \frac{1}{2}
    \sum_i
    \sigma_i^2
    \Tr\!\left(
        H_{ii}\operatorname{diag}(w_i^2)
    \right)
    +
    O\!\left(\mathbb{E}\|\delta W\|^3\right).
\end{equation}
The mean squared error is $ \displaystyle
    \MSE(W_i+\delta W_i,W_i)
=    \frac{1}{d_i}
    \mathbb{E}\|\delta w_i\|_2^2=    \frac{\sigma_i^2}{d_i}\|w_i\|_2^2$.

Thus, assuming $\|w_i\|_2>0$, we define $\displaystyle
    \alpha_i^{\mathrm{FP}}
    :=
    \frac{d_i}{2\|w_i\|_2^2}
    \Tr\!\left(
        H_{ii}\operatorname{diag}(w_i^2)
    \right)$ ,
and in order to again express the expected loss change as
\begin{equation}
\label{eq:multiplicative-mse-loss}
    \mathbb{E}\!\left[L(W+\delta W)-L(W)\right]
    =
    \sum_i
    \alpha_i^{\mathrm{FP}}\,
    \MSE(W_i+\delta W_i,W_i)
    +
    O\!\left(\mathbb{E}\|\delta W\|^3\right).
\end{equation}
Notice that
$    \Tr\!\left(
        H_{ii}\operatorname{diag}(w_i^2)
    \right)
    =
    \sum_{k=1}^{d_i}
    (H_{ii})_{kk}w_{i,k}^2,
$
so the corresponding sensitivity $\alpha_i^{\mathrm{FP}}$ is diagonal curvature weighted by squared magnitude of the weights.

\paragraph{Effect of error normalization.}
For the floating-point quantization noise, alternatively, one could define 
\begin{equation}
    \displaystyle \varepsilon_i :=\frac{\mathbb{E}||\delta w_i||^2}{||w_i||^2}, \; \alpha_i := d_i \Tr (H_{ii} \operatorname{diag} (w_i^2))
\end{equation}
 and obtain an expression similar to Eq.~\ref{eq:multiplicative-mse-loss}: $     \mathbb{E}\!\left[L(W+\delta W)-L(W)\right]
    \approx    \sum_i    \alpha_i^{}\, \varepsilon_i$.
We additionally investigated whether this alternative choice of normalization affects the performance of our method. 

In particular, we compared both the weight MSE, $\varepsilon_i = \displaystyle
    \frac{\|\delta W_i\|_F^2}{d_i}$,
and the relative squared error $\displaystyle \varepsilon'_i = \frac{\|\delta W_i\|_F^2}{\|W_i\|_F^2}$ following~\cite{malinovskii2025higgs}. In both cases $\alpha_i$ were computed folowing the procedure in Alg.~\ref{alg:method}. 

We find that the relative normalization $\varepsilon_i'$ gives consistently better final models for models trained with Muon, see Table~\ref{tab:error-normalization}. For AdamW the two choices are within noise for both 350M and 800M models. 

\begin{table}[t]
\centering
\small
\begin{tabular}{llccrccr}
\toprule
 & & \multicolumn{3}{c}{AdamW} & \multicolumn{3}{c}{Muon} \\
\cmidrule(lr){3-5}\cmidrule(lr){6-8}
Model & Budget & $\varepsilon = \MSE$ & $\varepsilon' = \frac{\|\delta W_i\|_F^2}{\|W_i\|_F^2}$ & $\Delta$ & $\varepsilon = \MSE$ & $\varepsilon' = \frac{\|\delta W_i\|_F^2}{\|W_i\|_F^2}$ & $\Delta$ \\
\midrule
350M         & 6.0 & \textbf{2.8598}{\scriptsize$\pm$.0014} & 2.8600{\scriptsize$\pm$.0003} & $+0.0002$
                     & 2.8434{\scriptsize$\pm$.0018} & \textbf{2.8392}{\scriptsize$\pm$.0014} & $-0.0041$ \\
                     \midrule
800M         & 5.0 & \textbf{2.6122} & 2.6124 & $+0.0002$ & 2.6105 & \textbf{2.6090} & $-0.0015$ \\
             & 6.0 & \textbf{2.6071} & 2.6087 & $+0.0016$ & 2.6080 & \textbf{2.6022} & $-0.0058$ \\
             \bottomrule
\end{tabular}
\caption{ Final test loss of \methodname mixed-precision runs with the error measured as absolute MSE vs.\ relative to the weight norm;
$\Delta = L(\varepsilon') - L(\varepsilon)$ . }
\label{tab:error-normalization}
\end{table}

Throughout the paper, we use MSE as the general notation for the error term in loss approximation and note that Muon benefits from normalizing by the squared weight norm.

\section{MSE Estimation: Analytic vs Measured}
\label{app:mse-estimation}
The estimate $\hatMSE$ can be obtained empirically by measuring the statistics of the actual weight tensors' quantization error. 
In practice, however, we rely on the efficient analytic approximation,
$\hatMSE(Q(W_i), W_i)\propto 4^{-b_i}$, based on the following analysis. 

For any bitwidth $b$ (except \text{BF16}, where no quantization is required), 
assume quantization with a signed uniform $b$-bit grid. Then, the quantization 
error of a value $x\in[-\xi,\xi)$ is 
\begin{equation}
\delta x=q(x)-x \qquad\text{for} 
\qquad q(x)=s\,\operatorname{clip}\bigl(\operatorname{round}(x/s),-2^{b-1},2^{b-1}-1\bigr),\label{eq:quant}
\end{equation}
with scale parameter $s=2^{-(b-1)}\xi$, \ie there are $2^b$ quantization intervals of size $s$. 
Assuming that values are distributed approximately uniformly within each quantization interval, the 
rounding error has variance $s^2/12$~\citep{6773026}.
Consequently, a natural estimate for the $\MSE$ is 
\begin{equation}
\mathbb{E}\|\delta X\|^2 \propto 4^{-b},
\label{eq:highres}
\end{equation}
recovering the expressions of \emph{high-resolution quantization}~\citep{gish1968asymptotically}. 
Note that the exact constants of proportionality are not important in our application, since they 
will be absorbed automatically in the estimate of $\alpha_i$. 

Employing a single error curve for NVFP4, MXFP6, MXFP8, and BF16 format constitutes a simplifying assumption, 
as these formats differ in element grid, block size, and scale representation. 
However, preliminary experiments using the measured relative weight error for each tier did not improve 
the allocation over Eq.~\eqref{eq:highres}, so we adopt the modeled curve throughout this study.%

\begin{table}[h]
\centering
\begin{tabular}{llccc}
\toprule
Model & Budget & Error & Final test loss & Improvement over FP4 \\
\midrule
\multirow{2}{*}{350M AdamW} & \multirow{2}{*}{6.33} & Analytic & $2.8598_{\pm 0.0014}$ & $0.527_{\pm 0.049}\,\%$ \\
                            &                       & Measured & $2.8598_{\pm 0.0017}$ & $0.527_{\pm 0.060}\,\%$ \\
\midrule
\multirow{2}{*}{800M AdamW} & \multirow{2}{*}{5.39} & Analytic & $\mathbf{2.6081}$ & $\mathbf{0.298}\,\%$ \\
                            &                       & Measured & $2.6095$ & $0.242\,\%$ \\
\bottomrule
\end{tabular}
\caption{Q-PACE with the analytic vs the measured quantization error for MSE estimation in Eq.~\eqref{eq:higgs} model.}
\label{tab:analytic_vs_measured}
\end{table}

\section{On loss linearity vs. forward-pass error}
\label{app:losslinear}

Table~\ref{tab:linearity} extends the analysis presented in Figure~\ref{fig:nmse-vs-loss} to all budget cells. Restricting the mixed tiers to the forward pass has minimal impact on the recovered share for any method, and \methodname achieves the highest performance in every cell.
Figure~\ref{fig:err-vs-loss} evaluates each model's final allocation using the sensitivity-weighted weight error $\sum_l \alpha_l e_l$ rather than the activation NMSE. The sensitivities $\alpha_l$ are refitted on the model's final weights to reflect the loss increase under weight noise.
Across both experimental settings and model sizes, the loss increase resulting from quantization, $L-L_{\mathrm{bf16}}$, exhibits a linear relationship with this error. This holds whether $e_l$ is modeled as $4^{-b_l}$ or measured directly on the quantized weights ($r\ge0.97$).
When using the modeled error, the fitted slope approaches one. In contrast, with the measured error, the slope is approximately one half.

\begin{figure}[p]
\centering
\begin{subfigure}{\linewidth}\centering
\includegraphics[width=0.7\linewidth,trim=0 55 0 0,clip]{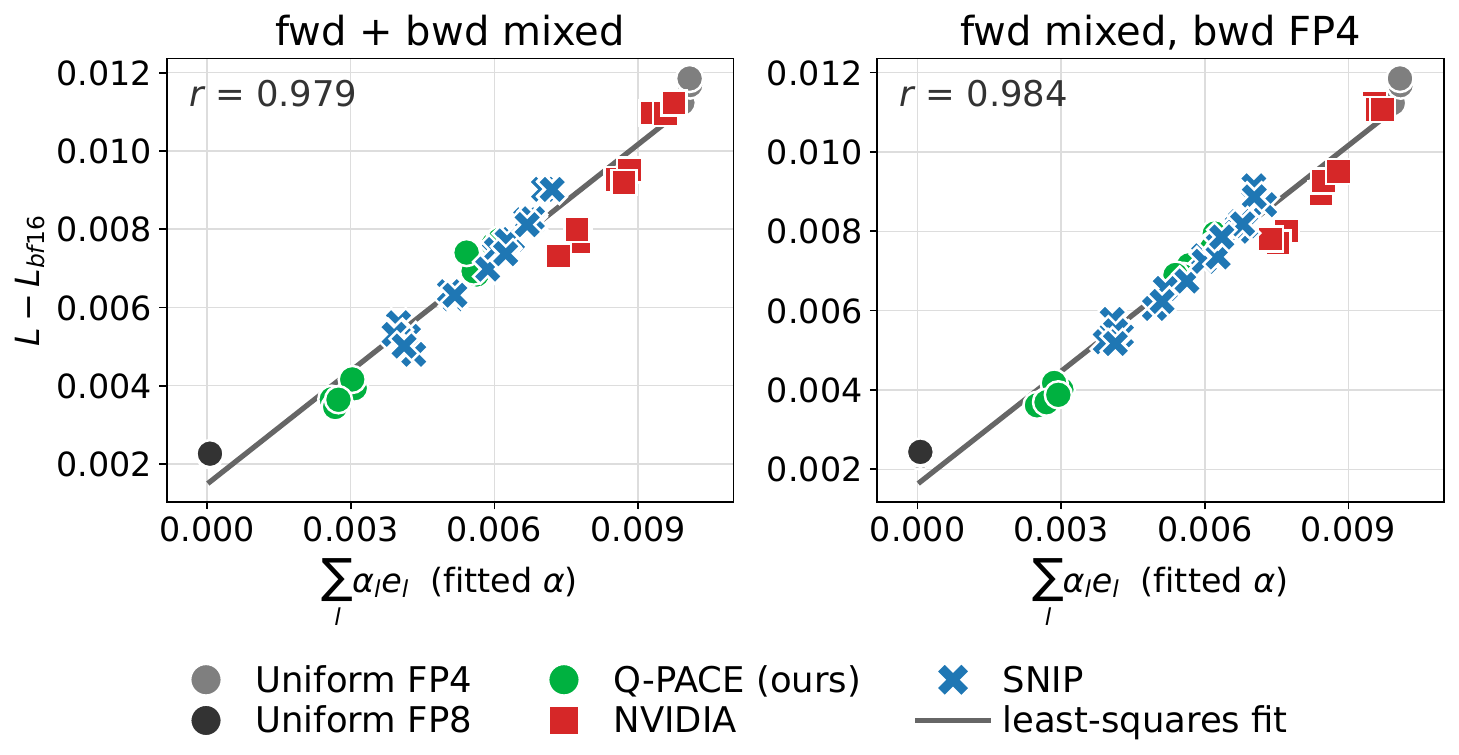}
\caption{350M, analytical error $e_l=4^{-b_l}$.}\label{fig:err-vs-loss-350m}
\end{subfigure}\\[1pt]
\begin{subfigure}{\linewidth}\centering
\includegraphics[width=0.7\linewidth,trim=0 55 0 0,clip]{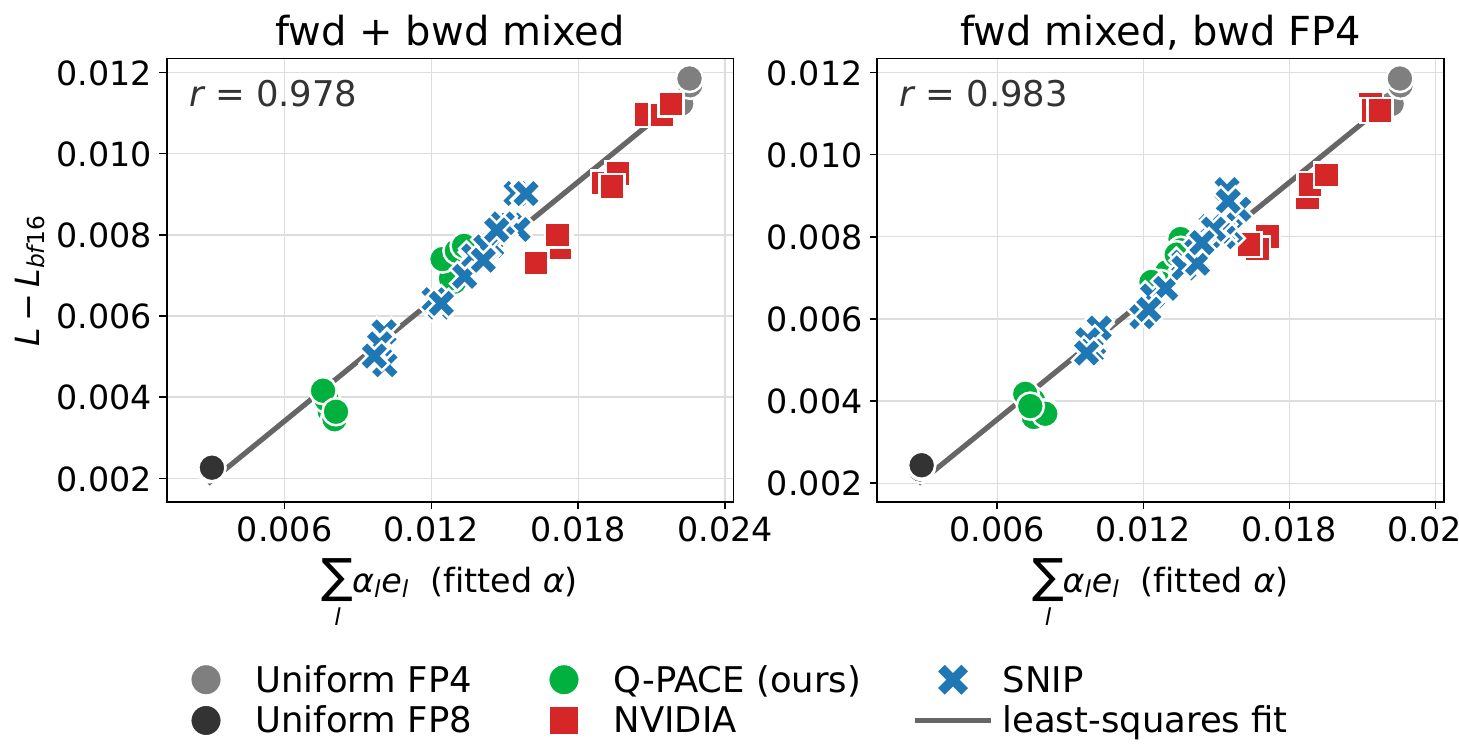}
\caption{350M, measured error $e_l(b_l)$.}\label{fig:err-vs-loss-350m-meas}
\end{subfigure}\\[1pt]
\begin{subfigure}{\linewidth}\centering
\includegraphics[width=0.7\linewidth,trim=0 55 0 0,clip]{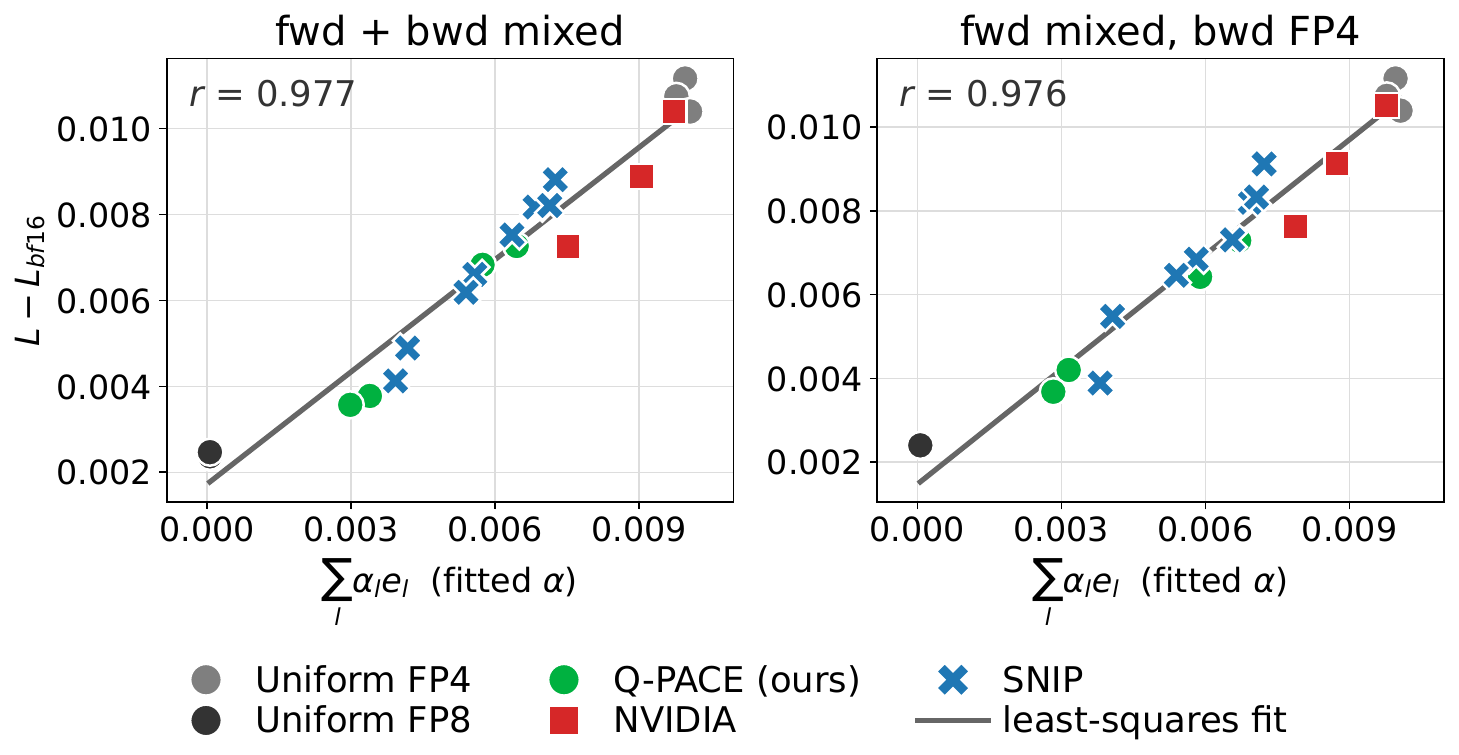}
\caption{500M, analytical error $e_l=4^{-b_l}$.}\label{fig:err-vs-loss-500m}
\end{subfigure}\\[1pt]
\begin{subfigure}{\linewidth}\centering
\includegraphics[width=0.7\linewidth,trim=0 55 0 0,clip]{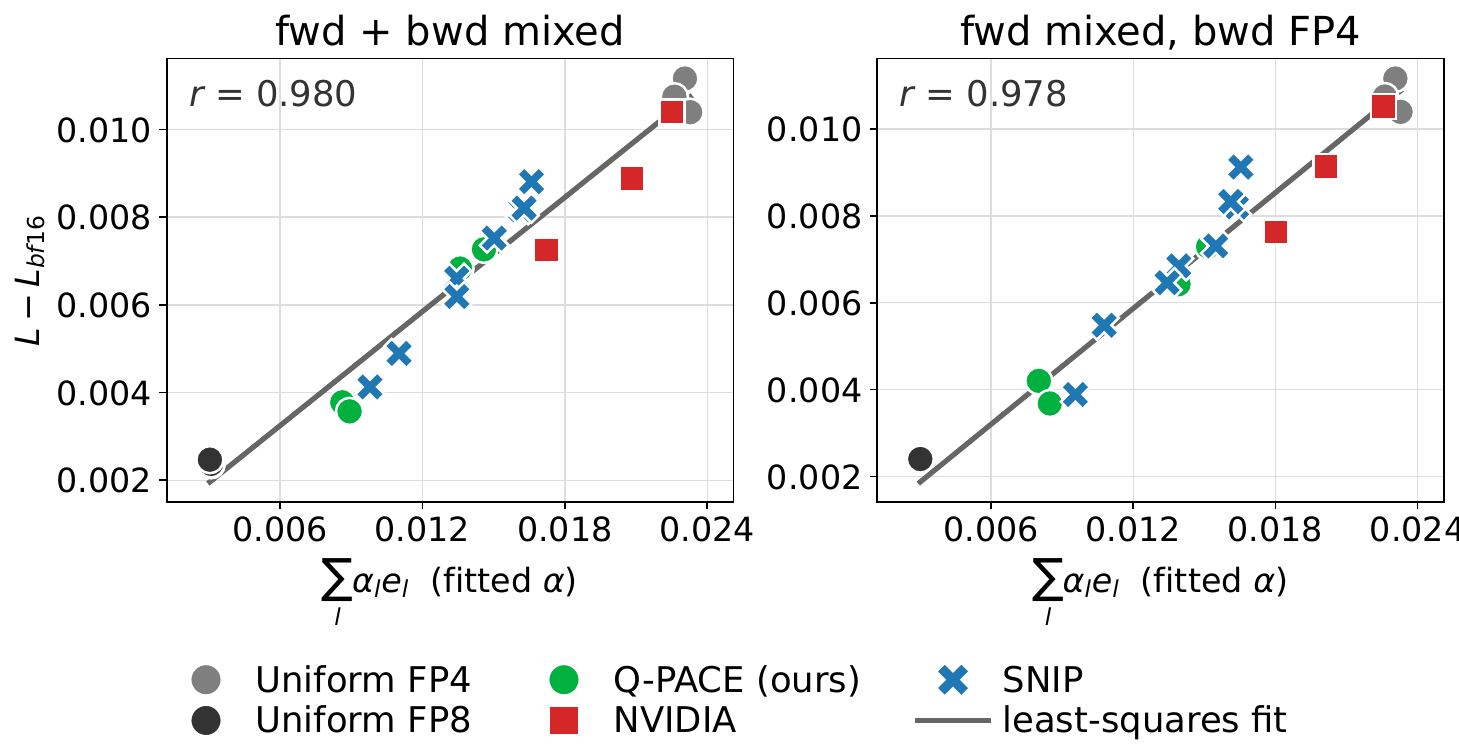}
\caption{500M, measured error $e_l(b_l)$.}\label{fig:err-vs-loss-500m-meas}
\end{subfigure}\\[1pt]
\includegraphics[width=0.7\linewidth,trim=0 0 0 307.5,clip]{figures/higgs_loss_500m_measured.pdf}
\caption{Loss increase from quantization, $L-L_{\mathrm{bf16}}$ (test loss minus that of the same weights in BF16), against the sensitivity-weighted weight error $\sum_l \alpha_l e_l$ of each model's final allocation, with $\alpha_l$ refitted on its final weights. Left: fwd + bwd mixed; right: fwd mixed, bwd FP4.}
\label{fig:err-vs-loss}
\end{figure}

\begin{table}[htbp]
\centering\scriptsize
\caption{Pretraining loss and recovered share of the NVFP4-to-MXFP8 gap (rec.) per budget cell when the tiers set both passes and when they set the forward only (FP4 backward everywhere; each setting against its own MXFP8 run).}
\label{tab:linearity}
\begin{tabular}{@{}llcccc@{}}
\toprule
 & & \multicolumn{2}{c}{fwd mixed + bwd mixed} & \multicolumn{2}{c}{fwd mixed, bwd fp4} \\
\cmidrule(lr){3-4}\cmidrule(lr){5-6}
Budget & Method & loss & rec. & loss & rec. \\
\midrule
\multicolumn{6}{l}{\textit{350M: uniform NVFP4 2.8864; uniform MXFP8 2.8480 (fwd mixed + bwd mixed), 2.8545 (fwd mixed, bwd fp4)}} \\
5 $\{4,6,8,16\}$ & \methodname & $\mathbf{2.8677}$ & $\mathbf{49\%}$ & $\mathbf{2.8703}$ & $\mathbf{50\%}$ \\
 & SNIP & 2.8724 & 37\% & 2.8735 & 41\% \\
 & SNIP (w-g-only) & 2.8700 & 43\% & 2.8728 & 43\% \\
 & NVIDIA & 2.8839 & 6\% & 2.8848 & 5\% \\
\cmidrule(lr){1-6}
5.5 $\{4,6,8,16\}$ & \methodname & $\mathbf{2.8581}$ & $\mathbf{74\%}$ & $\mathbf{2.8621}$ & $\mathbf{76\%}$ \\
 & SNIP & 2.8668 & 51\% & 2.8689 & 55\% \\
 & SNIP (w-g-only) & 2.8604 & 68\% & 2.8643 & 69\% \\
 & NVIDIA & 2.8805 & 15\% & 2.8818 & 15\% \\
\cmidrule(lr){1-6}
6.5 $\{4,8,16\}$ & \methodname & $\mathbf{2.8575}$ & $\mathbf{75\%}$ & $\mathbf{2.8627}$ & $\mathbf{74\%}$ \\
 & SNIP & 2.8638 & 59\% & 2.8667 & 62\% \\
 & SNIP (w-g-only) & 2.8613 & 65\% & 2.8648 & 68\% \\
 & NVIDIA & 2.8756 & 28\% & 2.8786 & 25\% \\
\cmidrule(lr){1-6}
6.5 $\{4,16\}$ & \methodname & $\mathbf{2.8697}$ & $\mathbf{43\%}$ & $\mathbf{2.8718}$ & $\mathbf{46\%}$ \\
 & SNIP & 2.8737 & 33\% & 2.8749 & 36\% \\
 & SNIP (w-g-only) & 2.8715 & 39\% & 2.8736 & 40\% \\
 & NVIDIA & 2.8756 & 28\% & 2.8786 & 25\% \\
\midrule
\multicolumn{6}{l}{\textit{500M: uniform NVFP4 2.6926; uniform MXFP8 2.6565 (fwd mixed + bwd mixed), 2.6634 (fwd mixed, bwd fp4)}} \\
5 $\{4,6,8,16\}$ & \methodname & $\mathbf{2.6727}$ & $\mathbf{55\%}$ & $\mathbf{2.6765}$ & $\mathbf{55\%}$ \\
 & SNIP & 2.6781 & 40\% & 2.6796 & 45\% \\
 & SNIP (w-g-only) & 2.6746 & 50\% & 2.6777 & 51\% \\
 & NVIDIA & 2.6920 & 2\% & 2.6921 & 2\% \\
\cmidrule(lr){1-6}
5.5 $\{4,6,8,16\}$ & \methodname & $\mathbf{2.6636}$ & $\mathbf{80\%}$ & $\mathbf{2.6688}$ & $\mathbf{82\%}$ \\
 & SNIP & 2.6729 & 54\% & 2.6759 & 57\% \\
 & SNIP (w-g-only) & 2.6652 & 76\% & 2.6694 & 80\% \\
 & NVIDIA & 2.6876 & 14\% & 2.6896 & 10\% \\
\cmidrule(lr){1-6}
6.5 $\{4,8,16\}$ & \methodname & $\mathbf{2.6635}$ & $\mathbf{81\%}$ & $\mathbf{2.6684}$ & $\mathbf{83\%}$ \\
 & SNIP & 2.6715 & 58\% & 2.6744 & 62\% \\
 & SNIP (w-g-only) & 2.6669 & 71\% & 2.6717 & 72\% \\
 & NVIDIA & 2.6833 & 26\% & 2.6859 & 23\% \\
\cmidrule(lr){1-6}
6.5 $\{4,16\}$ & \methodname & $\mathbf{2.6754}$ & $\mathbf{48\%}$ & $\mathbf{2.6786}$ & $\mathbf{48\%}$ \\
 & SNIP & 2.6797 & 36\% & 2.6809 & 40\% \\
 & SNIP (w-g-only) & 2.6783 & 40\% & 2.6794 & 45\% \\
 & NVIDIA & 2.6833 & 26\% & 2.6859 & 23\% \\
\bottomrule
\end{tabular}

\end{table}

\section{Training Setup}
\label{app:training-setup}
\paragraph{Pretraining.} The pretraining process employs a 32,000-token TokenMonster vocabulary, untied embeddings (tied in the 800M model), and Rotary Position Embedding (RoPE) with $\theta=10^4$.
Following RoPE, queries and keys are normalized, and a learned temperature parameter is applied to each attention head.
Each attention and multilayer perceptron (MLP) branch undergoes RMS normalization prior to residual addition, and a final RMSNorm is applied before the output head.
Table~\ref{tab:config} provides a summary of each model training configuration.
The 350M, 500M, and 4B models are trained with Muon on the weight matrices of the transformer blocks and AdamW on all other parameters.
The Muon optimizer utilizes a momentum of 0.95 and scales each orthogonalized $m\times n$ update by $0.2\sqrt{\max(m,n)}$. The AdamW optimizer employs $\beta=(0.9,0.95)$ and $\epsilon=10^{-6}$ \citep{muon,liu2025muon}.
The 800M model is trained with AdamW alone, with $\beta=(0.9,0.95)$, $\epsilon=10^{-8}$, and a weight decay of 0.1 on every parameter, including the embeddings.

\begin{table}[htbp]\centering\footnotesize\caption{Pretraining configurations.}
\label{tab:config}
\setlength{\tabcolsep}{3pt}\renewcommand{\arraystretch}{1.12}
\begin{tabularx}{\linewidth}{@{}>{\raggedright\arraybackslash}p{0.24\linewidth}*{4}{>{\raggedright\arraybackslash}X}@{}}\toprule
 & 350M & 500M & 800M & 4B\\\midrule
blocks / heads / GQA ratio (KV heads) & 11 / 12 / 4 (3) & 17 / 12 / 3 (4) & 17 / 16 / 4 (4) & 28 / 28 / 4 (7)\\
hidden / head dim & 1536 / 128 & 1536 / 128 & 2048 / 128 & 3584 / 128\\
parameters (total / non-embedding) & 370.8M / 272.5M & 526.2M / 427.9M & 814.3M / 748.8M & 4.006B / 3.777B\\
switchable linears & 66 & 102 & 102 & 168\\
context & 1024 & 1024 & 1024 & 4096\\
global batch (seqs / tokens) & 512 / 524k & 512 / 524k & 1024 / 1.05M & 256 / 1.05M\\
steps / tokens & 30k / 15.7B & 140k / 73.4B & 76.3k / 80.0B & 100k / 104.9B\\
optimizer & Muon + AdamW & Muon + AdamW & AdamW & Muon + AdamW\\
master weights & fp32 & fp32 & fp32 & fp32\\
peak LR / weight decay & $6{\times}10^{-3}$ / 0.01 & $6{\times}10^{-3}$ / 0.01 & $7.5{\times}10^{-4}$ / 0.1 & $2{\times}10^{-3}$ / 0.02\\
embedding weight decay & 0 & 0 & 0.1 & 0\\
schedule & WSD, cooldown 25k$\to$30k & WSD 5\%/20\% & WSD 5\%/20\% & warmup 2000, cooldown 20000\\
recalibration period & 4,000 & 10,000 & 2,000 & 10,000\\
seeds & 3 & 1 (mixed) / 3 (uniform) & 1 & 1\\
\bottomrule
\end{tabularx}
\end{table}

\paragraph{Supervised-Fine-tuning (SFT).}
Qwen3-4B-Base \citep{qwen3}, which includes 252 switchable projections, is fine-tuned on 10\% of the T\"ulu-3 dataset \citep{tulu3} using TRL \citep{trl}.
The Fused AdamW optimizer is configured with a learning rate of $5\times10^{-6}$, no weight decay, a linear-decay learning-rate scheduler with 3\% warmup, and gradient clipping set to 1.0.
Each epoch consists of 706 steps, with a global batch size of 128, a maximum sequence length of 4096, and training conducted in BF16 precision.
The chat template includes an empty thinking block, and the loss function is computed using only assistant tokens. SFT recalibrates every 128 steps, resulting in five allocations over 706 steps. Reallocation is disabled during the final pretraining step, ensuring that each run concludes under an allocation it has previously trained on. Each calibration step samples a new batch from the reserved calibration split.

\paragraph{Numerics and hardware.}
MXFP8 employs E4M3 values with 32-element E8M0 scales, round-to-nearest (RTN) rounding, and a ceiling scale rule. MXFP6 utilizes E2M3 with RTN for the forward pass and stochastic rounding (SR) for the backward pass. NVFP4 adopts the Quartet~II path, which incorporates randomized Hadamard transforms and its unbiased backward estimator \citep{panferov2026quartet}. In this implementation, MXFP6 values are represented exactly in E4M3 and processed through the MXFP8 GEMM path. Consequently, the 6-bit tier evaluates 6-bit numerics without assessing a native 6-bit execution path. All experiments are conducted on NVIDIA B200 GPUs, which natively support low-bit data formats.

\clearpage

\section{SNIP Weight-grad-only} 
\label{app:snipwgo}
SNIP~\citep{snip_quant} provides a first-order alternative to our method. Let $G_l$ be the full-precision gradient of layer $l$ on the calibration batch and let $\delta W_l^{(b)}=q_b(W_l)-W_l$ be its measured quantization error at tier $b$. We use
\begin{equation}
q_{l,b}=
\frac{\|G_l\|_F\|\delta W_l^{(b)}\|_F}{\sqrt{d_l}},
\label{eq:snip_}
\end{equation}
which bounds the magnitude of the first-order loss change up to the stated normalization. The same budgeted solver minimizes the selected scores. The original SNIP framework also includes activation- and optimizer-propagated backward terms. We observed that using only Eq.~\ref{eq:snip_} outperforms the full pipeline. Therefore, we added it to some of the experiments for comparison and call it SNIP weight-grad-only.

\section{\methodname hyperparameter ablation}
\label{app:ablation-m-D}
Here we evaluate how the performance of \methodname depends on the two
main hyperparameters: the number of noise levels $m$ used to estimate each layer sensitivity, and the size of the calibration dataset $|\mathcal{D}|$.

Figure~\ref{fig:ablation-heatmap-m-D} reports the final test-loss
improvement for 350M model over uniform FP4 training for different combinations of
$m$ and $|\mathcal{D}|$. Using only a single perturbation and a single
calibration sequence leads to the least accurate allocation. Increasing
either the number of noise levels or the amount of calibration data
improves the resulting allocation, but the gains quickly saturate.

This indicates that \methodname does not require a highly accurate
curvature estimate to identify useful relative differences in
layer sensitivity. A modest number of perturbations $(m \geq 4)$ and calibration
examples $|\mathcal{D}| \geq 64$ is already sufficient to obtain effective precision
allocations. This robustness to calibration
hyperparameters allows us to keep the computational overhead small
during training.

\begin{figure}[h]
    \centering
    \includegraphics[width=0.5\linewidth]{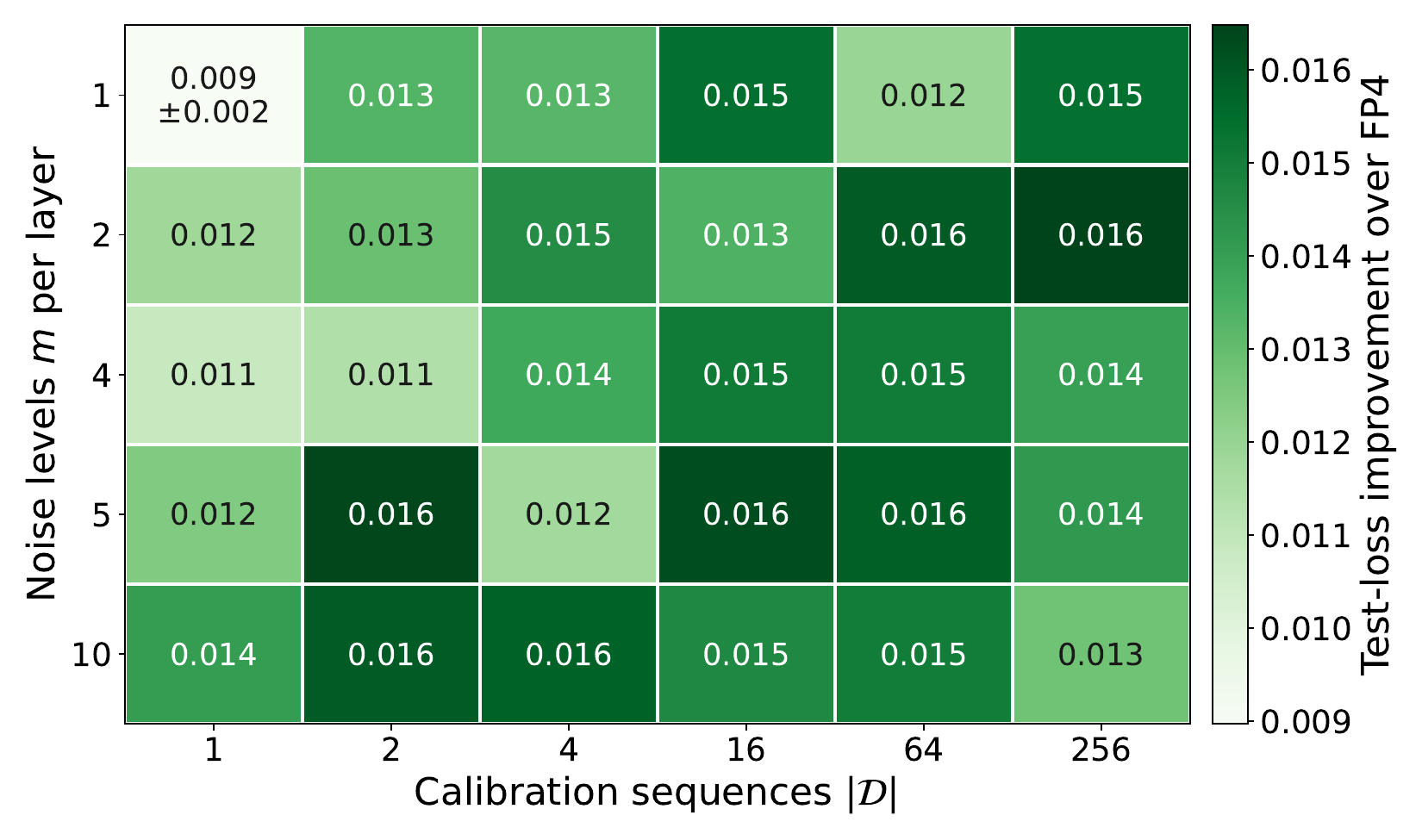}
    \caption{Test loss gain over uniform FP4 for 350M model over the number of noise levels $m$ per layer and calibration sequences $|\mathcal{D}|$.  Performance
    improves over the smallest calibration setting (1,1) and then remains
    relatively stable across a broad range of $m$ and $|\mathcal{D}|$,
    indicating diminishing returns from additional calibration compute.}
    \label{fig:ablation-heatmap-m-D}
\end{figure}

\section{Alternative sensitivity models}

\methodname is primary sensitivity model. We additionally compare against several alternative approximations.

\paragraph{\methodname inversed.} Alternatively to \methodname model, one can consider a symmetric linear approximation relative to the loss of uniformly FP4 quantized model:
\begin{equation}
    L(W, B) = L_{fp4} - \sum_i \alpha_i \frac{1}{d_i} ||\delta W_i||^2_F ,
\end{equation}

here $\Delta W_i = 0$ for quantized weights, and the sum effectively covers the set  of \textbf{non-quantized} weights. $ L_{fp4} = L(W, (b_i = fp4))$ is the loss evaluated on fully quantized model weights.

Generally, we want to assign each linear layer a scalar value to measure the "sensitivity" of a particular layer, that is, how much the loss would change if the layer were quantized while keeping the remaining layers fixed. 

\begin{equation}
    L(W, B) = \sum_i f(\alpha_i) + L_{bf16}.
\end{equation}

\paragraph{First-order Taylor sensitivity.} By keeping only the linear term, one can obtain the first-order approximation, which relates to the quality metric used by \citep{snip_quant}.

\begin{equation}
    \alpha_i = \langle \nabla_{W_i}, \Delta W_i \rangle, \;\;\;
         L(W, B) = \sum_i \alpha_i+ L_{bf16}.
\end{equation}

\paragraph{Cosine similarity.}

\begin{equation}
     \alpha_i = 1 - \cos (W_i, Q(W_i)) = 1 - \frac{\langle W_i, Q(W_i) \rangle}{||W_i||_F ||Q(W_i)||_F}.
\end{equation}

\paragraph{\methodname-Quadratic model.}
 \methodname formulation considers only independent effects of each layer. It can be extended to include pair-wise interactions between  layers:

\begin{equation}
        L(W, B) = \sum_i f(\alpha_i) + \sum_{i\neq j} g(\gamma_{ij}) +  L_{bf16}.
\end{equation}

\begin{equation}
    L(W, B) = \sum_k \alpha_k \frac{||\Delta W_k||^2_F}{||W_k||^2_F} + \sum_{k \neq j} \gamma_{kj} \frac{||\Delta W_k||_F||\Delta W_j||_F}{||W_k||_F ||W_j||_F} + L_{bf16}.
\end{equation}
where $\{ k \}$ is the set  of quantized weights. 

In Table~\ref{tab:alternative-methods}, we report the test loss results for 350M-parameter model runs using the alternative strategies described above. All evaluated strategies improve over the FP4 baseline, although the magnitude of the improvement varies. The best results are obtained by the inversed \methodname variant and the default \methodname, which achieve equal improvements within error interval. Variants based on measured weight error, gradient MSE, and the quadratic \methodname perform similarly, yielding improvements between $0.472\%$ and $0.491\%$. In contrast, cosine similarity, first-order Taylor, and \lastlayers provide smaller gains over FP4. 

Overall, these results show that the performance of \methodname is robust across several alternative formulations, while the default \methodname formulation remains the strongest variant.

\begin{table}[t]
\centering
\caption{Test loss results for 350M model trained on 18B tokens with budget 6.5 bit/parameter and choices {4,16}. Average and std error over 3 random seeds.}
\label{tab:alternative-methods}
\begin{tabular}{lcc}
\toprule
Method & Final validation loss & Improvement over FP4 \\
\midrule
\methodname inversed                           & $2.8551_{\pm 0.0010}$ & $0.554_{\pm 0.038}\,\%$ \\
\methodname                                 & $2.8554_{\pm 0.0011}$ & $0.544_{\pm 0.037}\,\%$ \\
\methodname: measured MSE($W$)                            & $2.8569_{\pm 0.0021}$ & $0.491_{\pm 0.061}\,\%$ \\
\methodname: MSE($W_\text{grad}$)            & $2.8570_{\pm 0.0005}$ & $0.488_{\pm 0.031}\,\%$ \\
Quadratic \methodname        & $2.8575_{\pm 0.0018}$ & $0.472_{\pm 0.050}\,\%$ \\
\methodname: weights + activations MSE          & $2.8581_{\pm 0.0010}$ & $0.450_{\pm 0.048}\,\%$ \\
Cosine sim.                           & $2.8593_{\pm 0.0039}$ & $0.410_{\pm 0.138}\,\%$ \\
First-order Taylor                   & $2.8628_{\pm 0.0015}$ & $0.280_{\pm 0.044}\,\%$ \\
\lastlayers                   & $2.8641_{\pm 0.0012}$ & $0.240_{\pm 0.046}\,\%$ \\
FP4                                   & $2.8710_{\pm 0.0004}$ & $0.00$ \\
\bottomrule
\end{tabular}
\end{table}

\section{Additional sensitivity analysis}
\label{app:sensitivity-analysis}
We report additional sensitivity results in this section for the 800M model. These results complement the analysis in the main text by showing how the sensitivity profile varies across layer types, model depth, and training time.

\begin{figure}[t]
    \centering
    \includegraphics[width=0.7\linewidth]{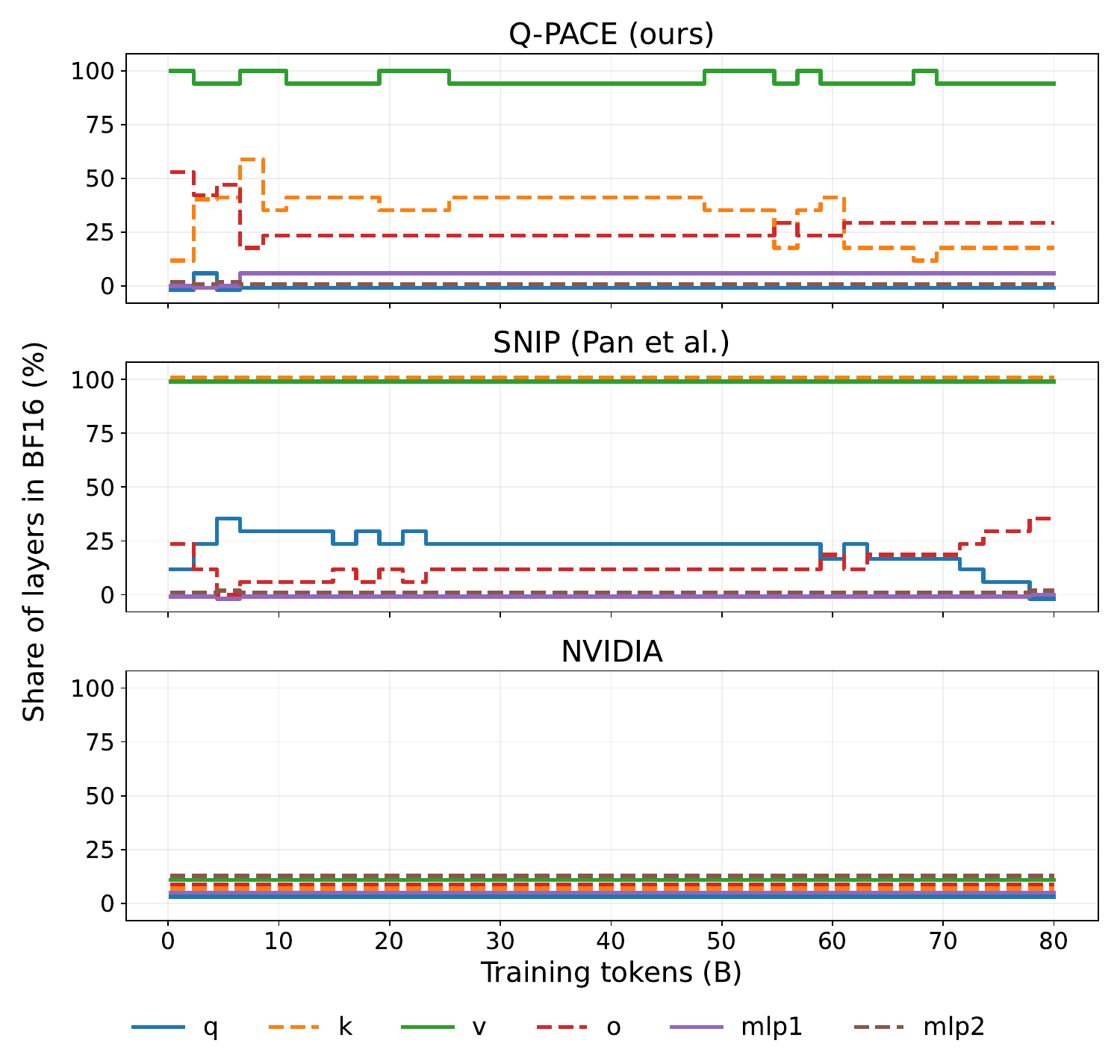}
    \caption{Share of different layer types promoted to BF16 over the training duration.}
    \label{fig:bf16-share-800m}
\end{figure}

Figure~\ref{fig:bf16-share-800m} shows the fraction of each layer type assigned
to BF16 throughout training. The allocations produced by \methodname are highly
non-uniform across layer types and evolve over the course of training. In
particular, the V projections are consistently assigned high precision, while
the fractions of K and O projections with high precision vary. In contrast, SNIP keeps nearly all K and V projections in BF16 for
most of training, while the NVIDIA baseline induces a static
allocation across layer types. 

This illustrates that \methodname adapts its
precision allocation over training time.

\begin{figure}[t]
    \centering
    \includegraphics[width=0.7\linewidth]{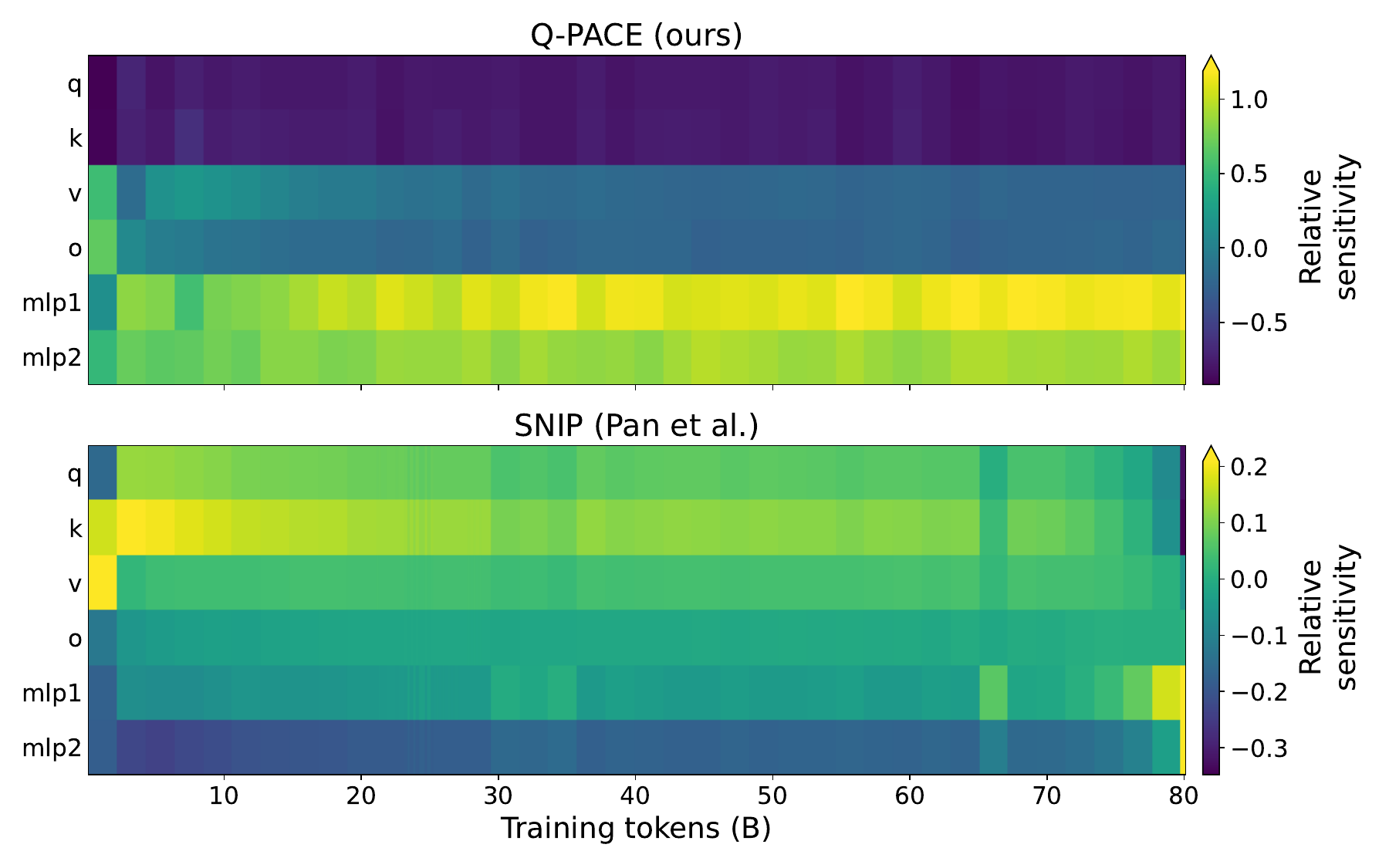}
    \caption{Relative sensitivity by layer type over training for the 800M model.
    Values are averaged across layers of the same type at each calibration
    point. \methodname assigns the largest relative sensitivity to the MLP
    projections, whereas SNIP places comparatively greater sensitivity on
    K projections.}
    \label{fig:sensitivity-type-800m}
\end{figure}

To better understand these allocations, Figure~\ref{fig:sensitivity-type-800m}
shows the corresponding relative sensitivity scores aggregated by layer type.
For \methodname, the MLP projections exhibit the largest relative sensitivities,
whereas attention projections, in particular Q and K, are substantially less
sensitive on average. SNIP produces a different ordering, assigning
higher relative sensitivity to attention projections and lower sensitivity to
the MLP projections. The sensitivity profiles are comparatively stable after
the early stages of training. 

This difference helps explain why the two methods can produce substantially
different mixed-precision allocations despite operating under the same global
precision budget.

\begin{figure}[t]
    \centering
    \includegraphics[width=0.8\linewidth]{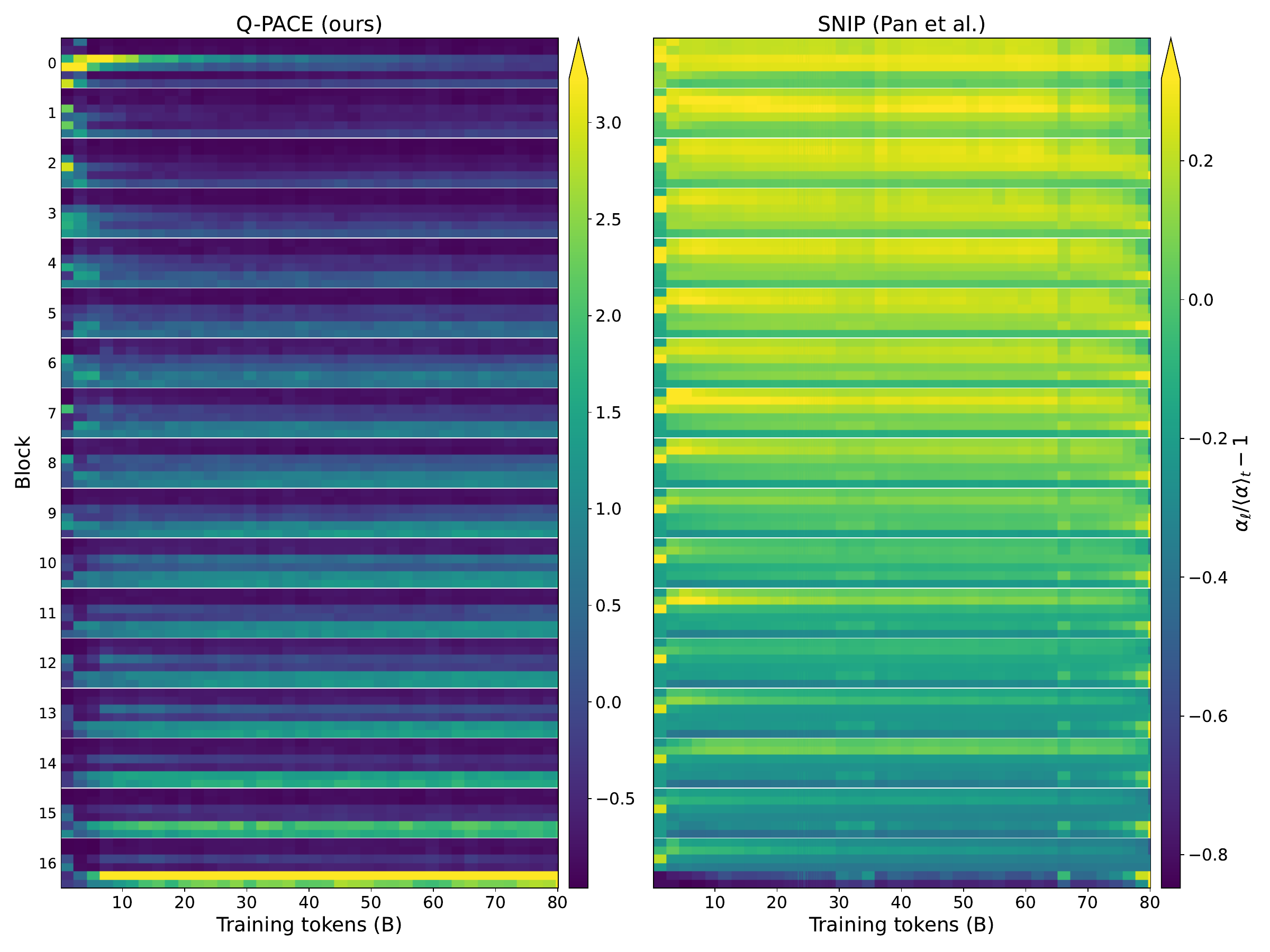}
    \caption{Relative sensitivity over training for the 800M model. }
    \label{fig:sensitivity-depth-800m}
\end{figure}

Sensitivity is also strongly structured by model depth.
Figure~\ref{fig:sensitivity-depth-800m} shows the sensitivity $\alpha_i$ normalized by the average sensitivity at that training point. For \methodname, later blocks are generally more sensitive, with the final block standing out throughout most of training.
After an initial transition, this pattern becomes remarkably stable. SNIP exhibits a qualitatively different depth profile, with higher relative sensitivity concentrated toward earlier blocks and decreasing toward the end of the network. These results further show that quantization sensitivity cannot be described only by layer type or only by depth: \methodname captures a structured combination of both factors that remains largely stable over long portions of training.

\begin{figure}[t]
    \centering
    \includegraphics[width=0.7\linewidth]{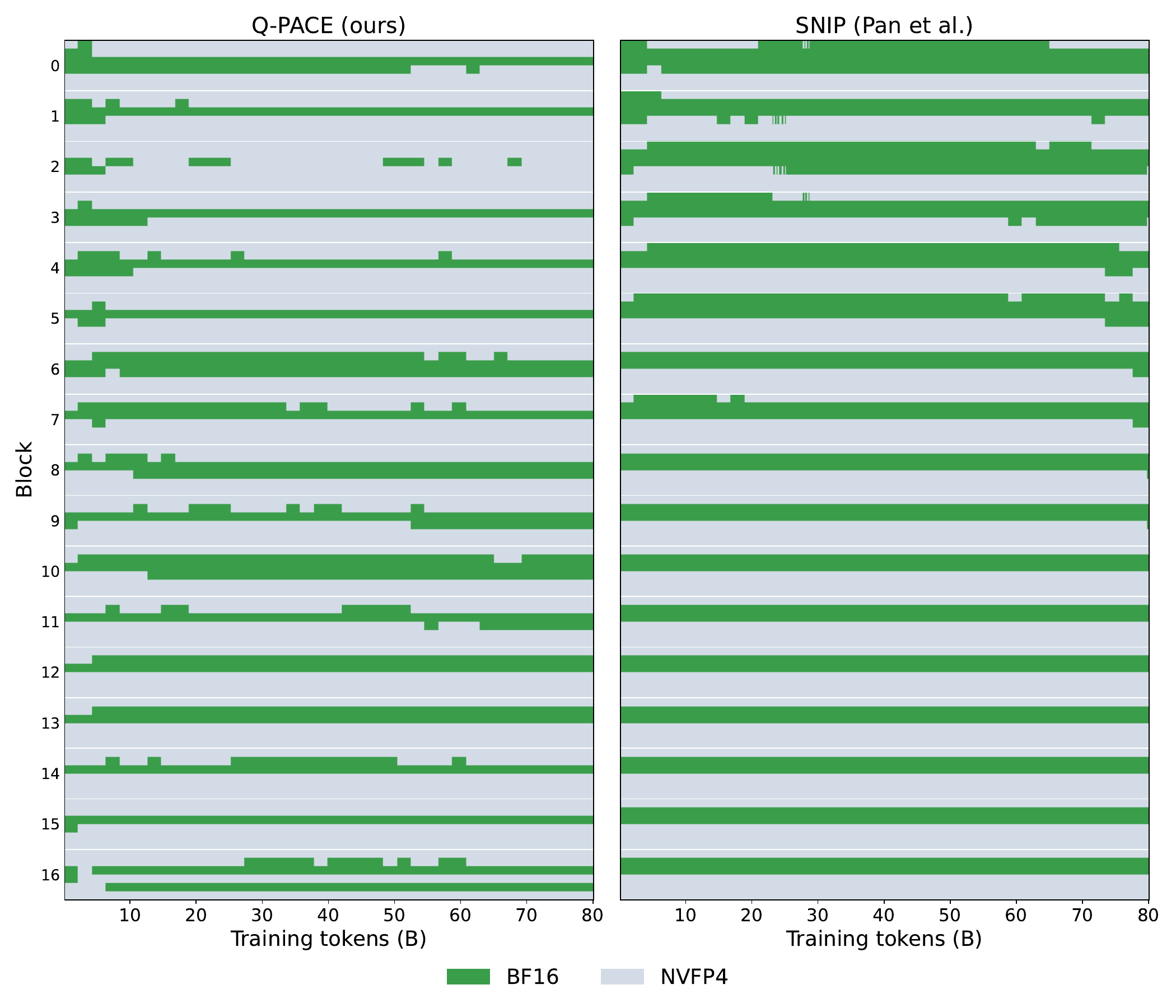}
    \caption{Bitwidth allocations throughout training for 800M model at budget 5.5 bits/parameter with bit choices {4,16}.}
    \label{fig:allocations-800m-app}
\end{figure}

\begin{figure}[t]
    \centering
    \includegraphics[width=0.8\linewidth]{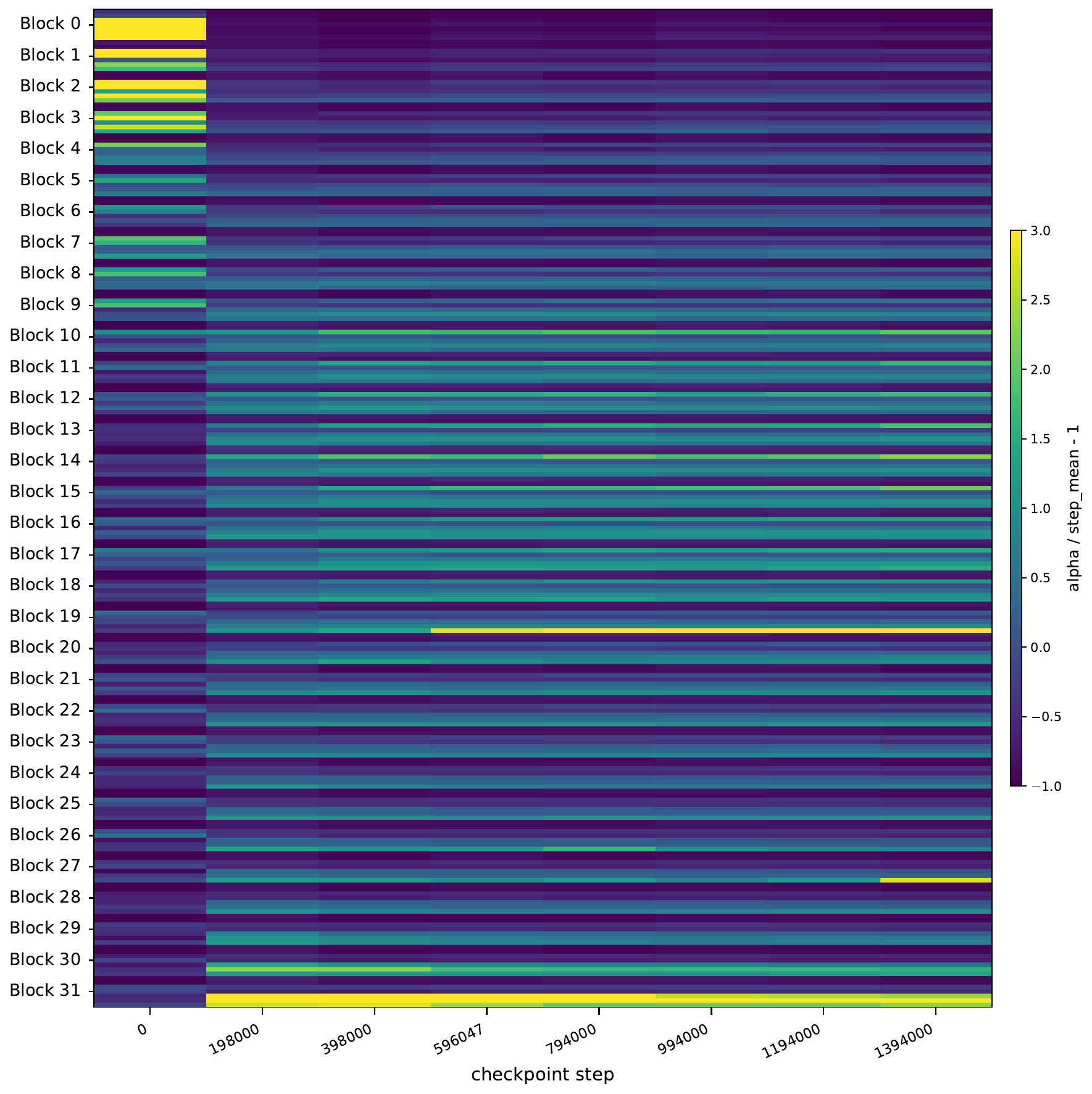}
    \caption{Layer-wise sensitivity heatmap for Olmo-3-7B~\citep{olmo3} during training. Lighter color corresponds to higher sensitivity $\alpha$.}
    \label{fig:olmo-heatmap}
\end{figure}

\clearpage
\section{Budget trade-off, sensitivity, and allocations}
\label{app:figures}
This section presents the figures supporting the grid and 4B results. Figure~\ref{fig:tradeoff} displays each allocation on the budget vs. loss plots. At both pretraining sizes and during SFT, \methodname achieves greater improvement over uniform NVFP4 than both SNIP variants and NVIDIA at a comparable bit count. The subsequent figures illustrate the distribution of precision. Sensitivity is normalized at each calibration, allowing for comparison over time but not across estimators (Figures~\ref{fig:sens} and~\ref{fig:sensft}). During pretraining, \methodname concentrates its score on the MLP up-projections, whereas SNIP distributes it evenly across projection types. In SFT, the score for \methodname shifts to the MLP down-projections. The allocations correspond to these scores and exhibit minimal change after the initial recalibration (Figures~\ref{fig:alloc} to~\ref{fig:alloc-sft}). Figures~\ref{fig:sens4b} to~\ref{fig:heatall} provide a parallel analysis for the 4B runs across all 168 switchable layers.
\begin{figure}[htbp]
\centering
\includegraphics[width=\textwidth]{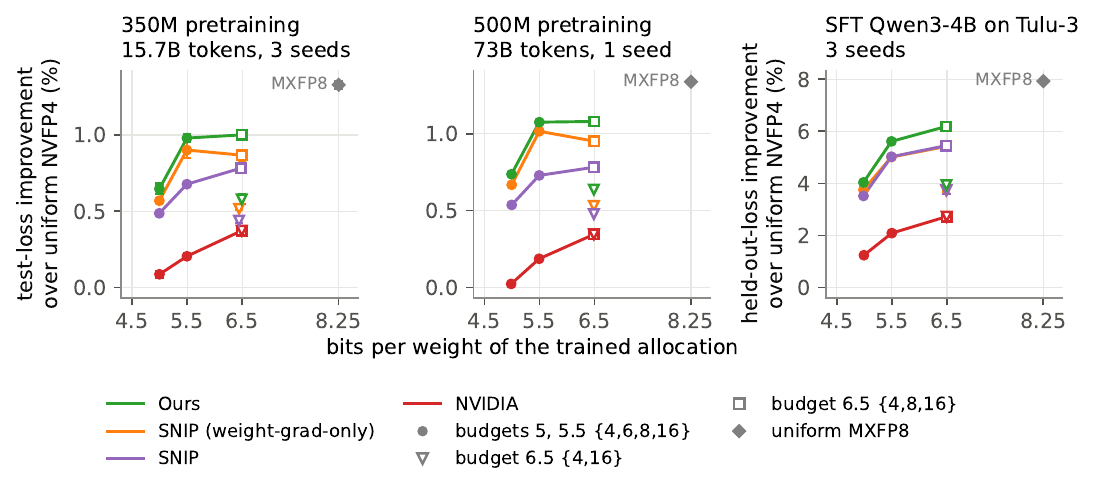}
\caption{Loss improvement over uniform NVFP4 against the realized bits per weight of the trained allocation, for pretraining at 350M and 500M and for SFT.}
\label{fig:tradeoff}\label{fig:tradeoff-sft}
\end{figure}

\begin{figure}[htbp]
\centering
\includegraphics[width=0.92\textwidth]{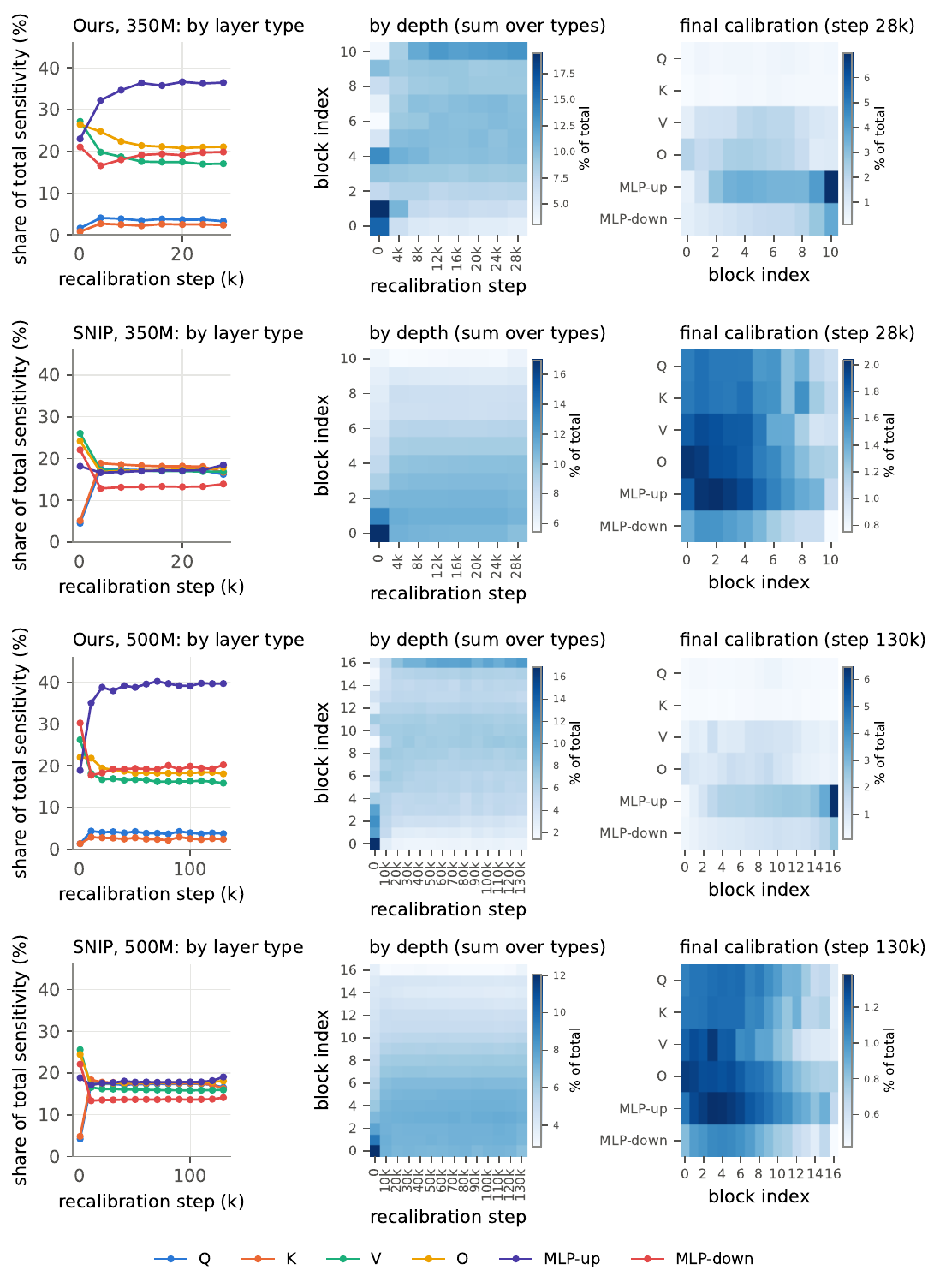}
\caption{Normalized sensitivity of \methodname and SNIP at budget 6.5 on $\{4,8,16\}$ at 350M and 500M. Type shares over the calibrations, depth shares, and the final type-by-block map.}
\label{fig:sens}
\end{figure}

\begin{figure}[htbp]
\centering
\includegraphics[width=0.95\textwidth]{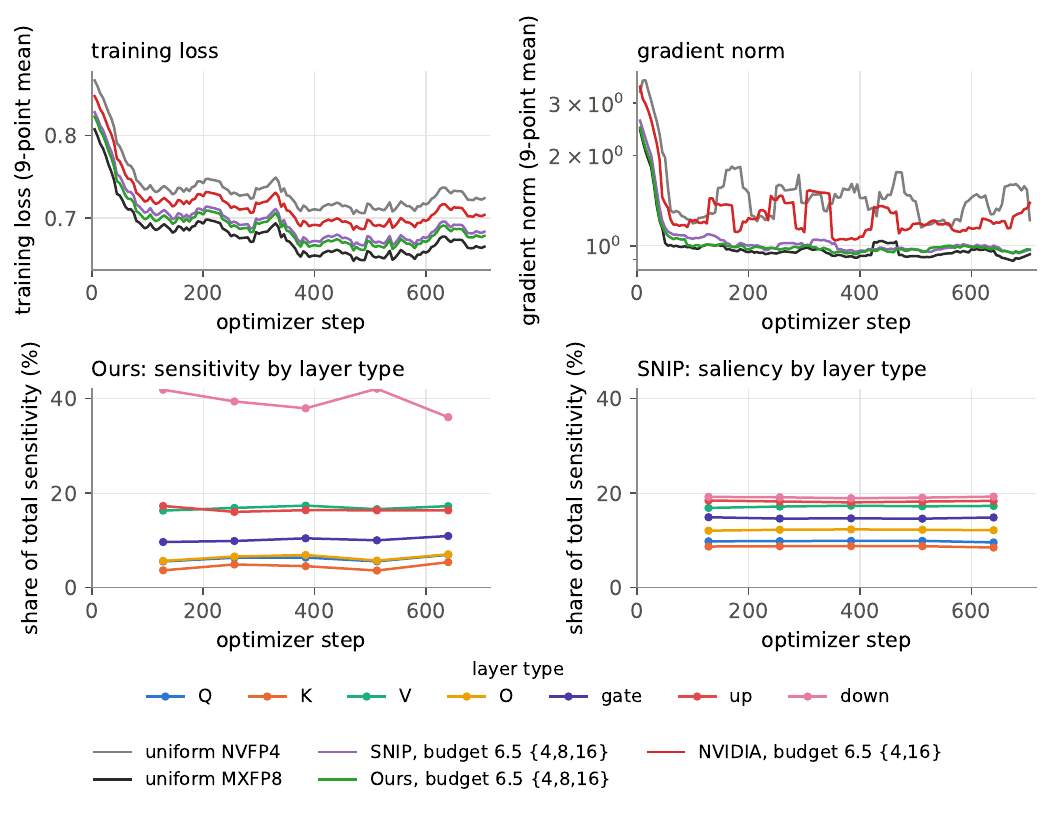}
\caption{SFT training loss and gradient norm (mean of three seeds) and the normalized sensitivity by projection type of \methodname and SNIP at budget 6.5 on $\{4,8,16\}$.}
\label{fig:sensft}
\end{figure}

\begin{figure}[htbp]
\centering
\includegraphics[width=0.92\textwidth]{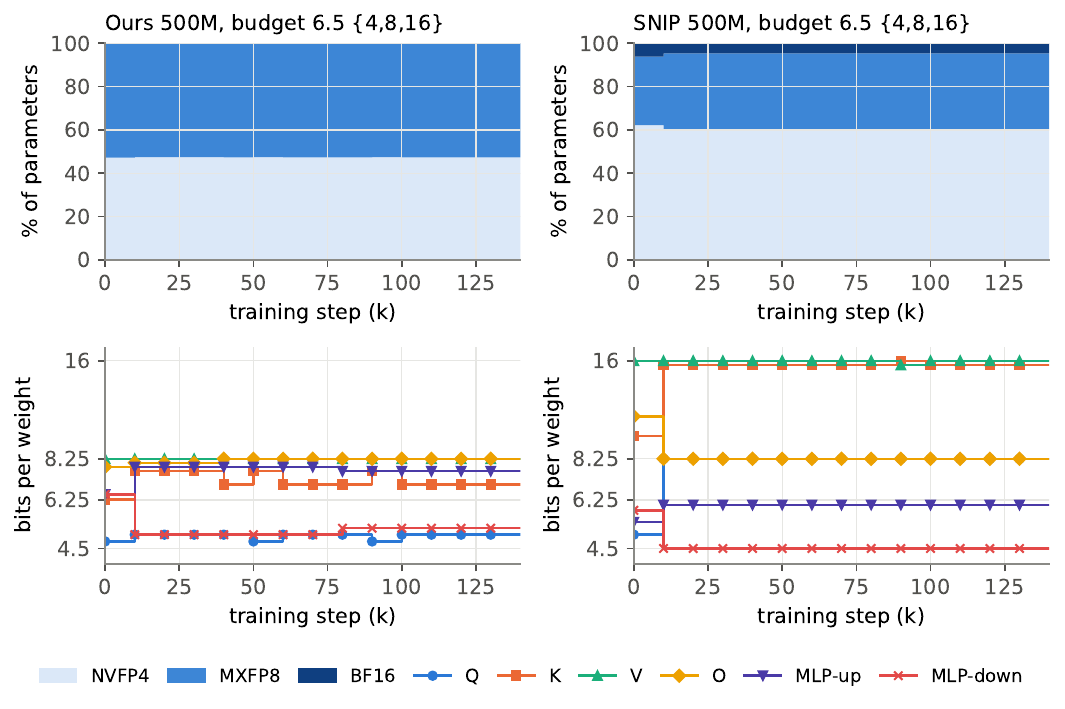}
\caption{500M allocations of \methodname and SNIP at budget 6.5 on $\{4,8,16\}$ over training. Parameter share per format and bits per weight by projection type.}
\label{fig:alloc}
\end{figure}

\begin{figure}[htbp]
\centering
\includegraphics[width=0.9\textwidth]{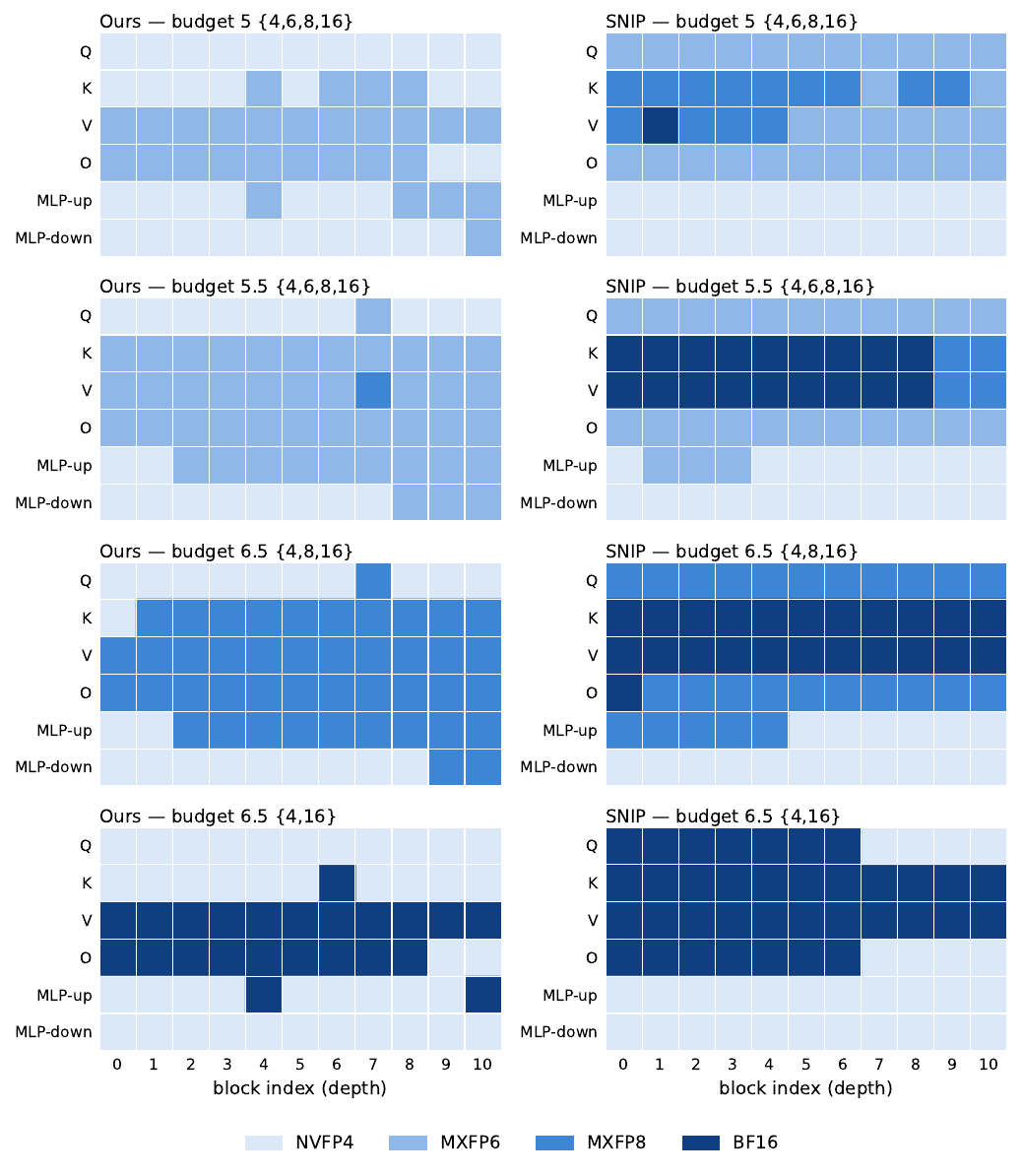}
\caption{Final 350M allocations of \methodname and SNIP for every budget and menu.}
\label{fig:alloc-pt350}
\end{figure}

\begin{figure}[htbp]
\centering
\includegraphics[width=0.9\textwidth]{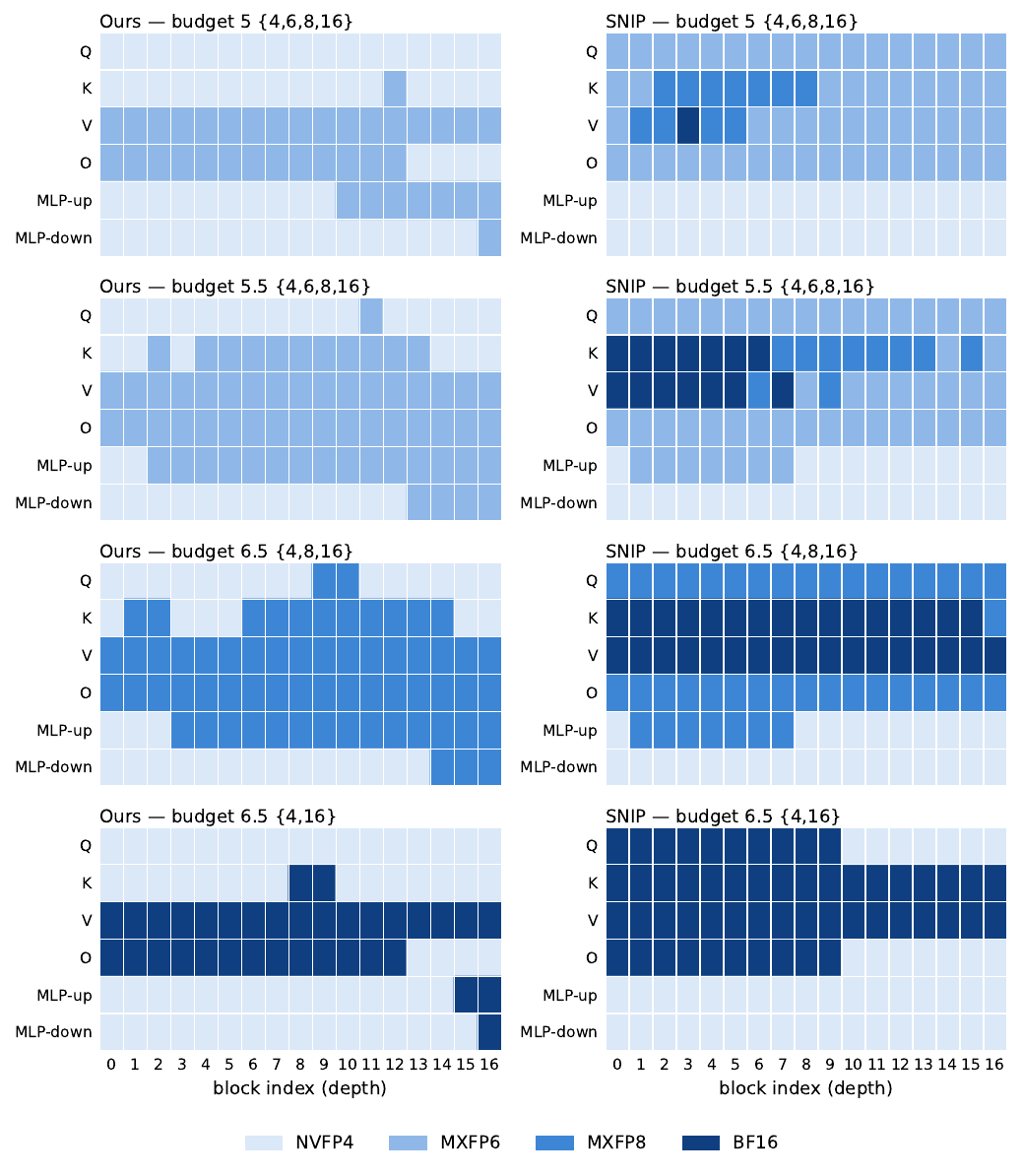}
\caption{Final 500M allocations of \methodname and SNIP for every budget and menu.}
\label{fig:alloc-pt05b}
\end{figure}

\begin{figure}[htbp]
\centering
\includegraphics[width=0.8\textwidth]{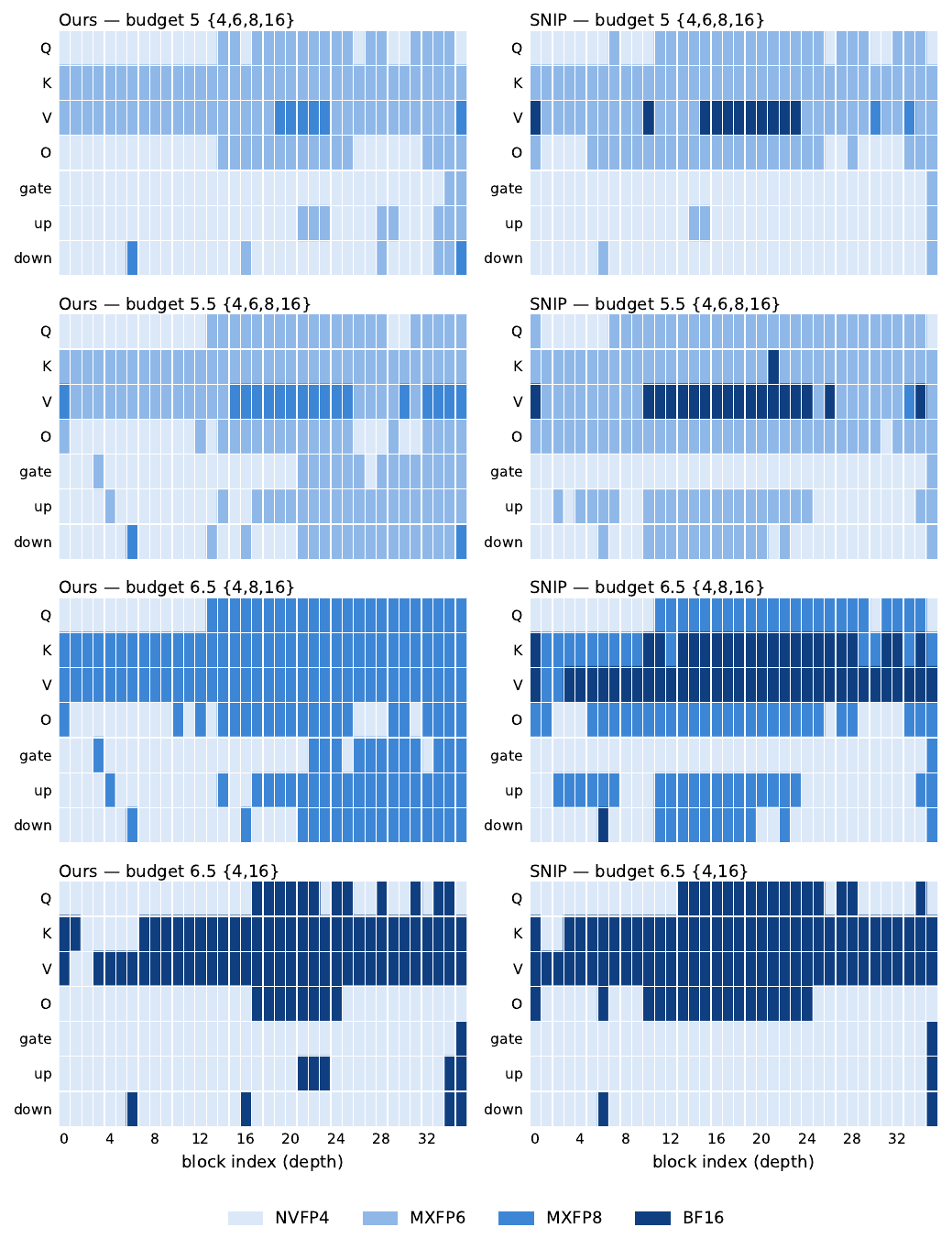}
\caption{Final SFT allocations of \methodname and SNIP for every budget and menu. Qwen3 has three MLP projections per block.}
\label{fig:alloc-sft}
\end{figure}

\begin{figure}[htbp]
\centering
\begin{subfigure}{\textwidth}\centering
\includegraphics[width=\linewidth]{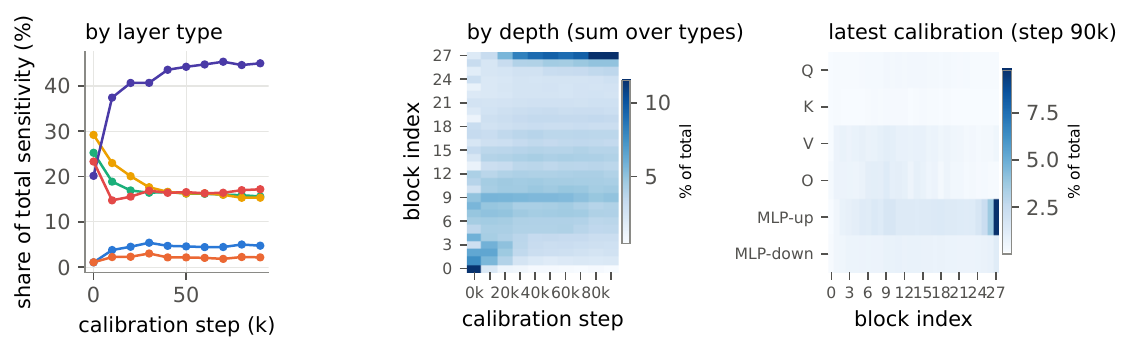}
\caption{\methodname}
\end{subfigure}\\[6pt]
\begin{subfigure}{\textwidth}\centering
\includegraphics[width=\linewidth]{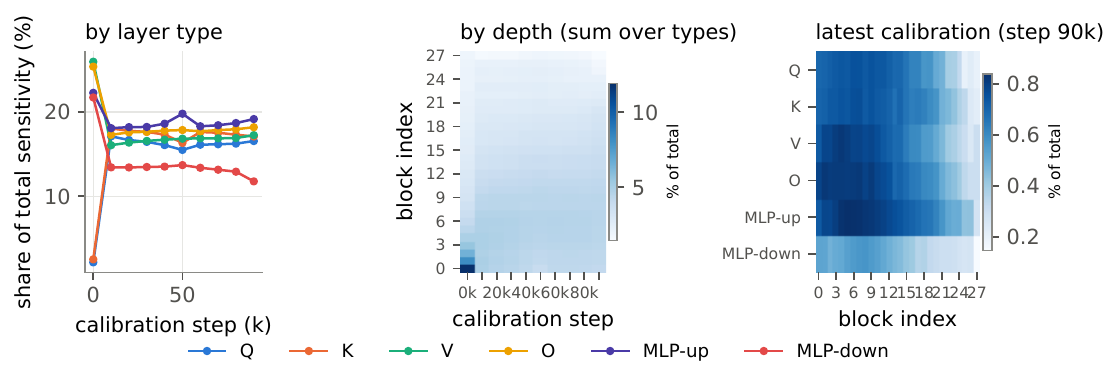}
\caption{SNIP}
\end{subfigure}
\caption{4B normalized sensitivity at budget 6.5 on $\{4,16\}$. Type shares over the calibrations, depth shares, and the latest type-by-block map.}
\label{fig:sens4b}
\end{figure}

\begin{figure}[htbp]
\centering
\begin{subfigure}[b]{0.46\textwidth}\centering
\includegraphics[height=4.3in]{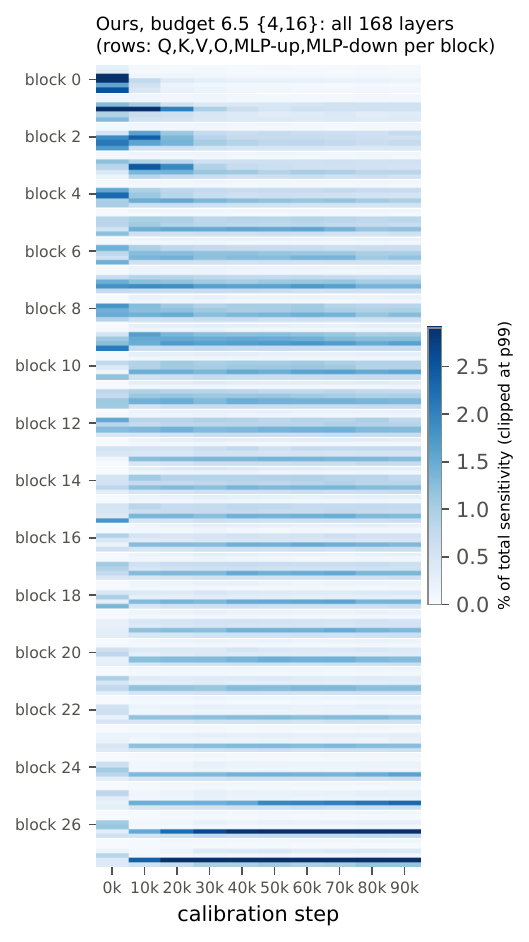}
\caption{\methodname}
\end{subfigure}\hfill
\begin{subfigure}[b]{0.5\textwidth}\centering
\includegraphics[height=4.3in]{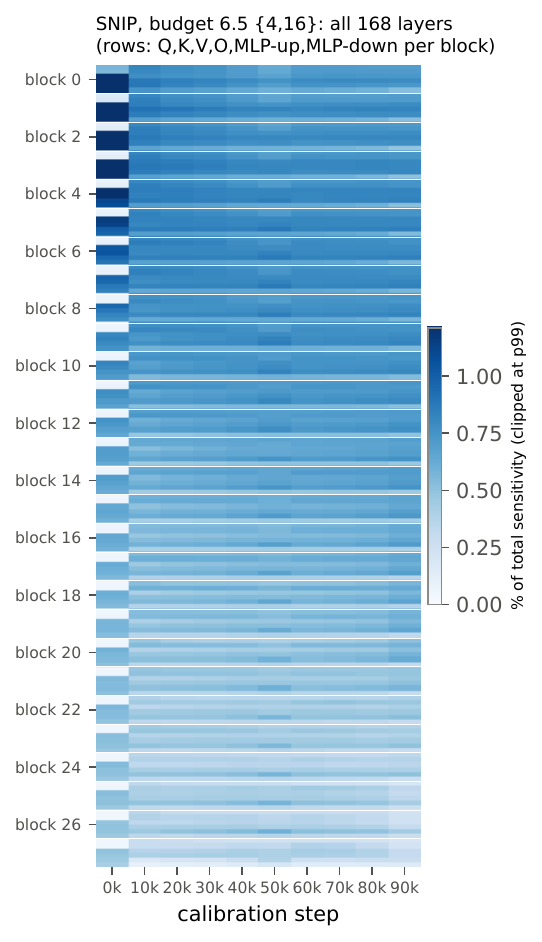}
\caption{SNIP}
\end{subfigure}
\caption{4B sensitivity of all 168 layers across the calibrations at budget 6.5 on $\{4,16\}$. Colors are clipped at the 99th percentile, on separate scales for the two estimators.}
\label{fig:sens4b-layers}
\end{figure}

\begin{figure}[htbp]
\centering
\begin{subfigure}{\textwidth}\centering
\includegraphics[width=0.8\linewidth]{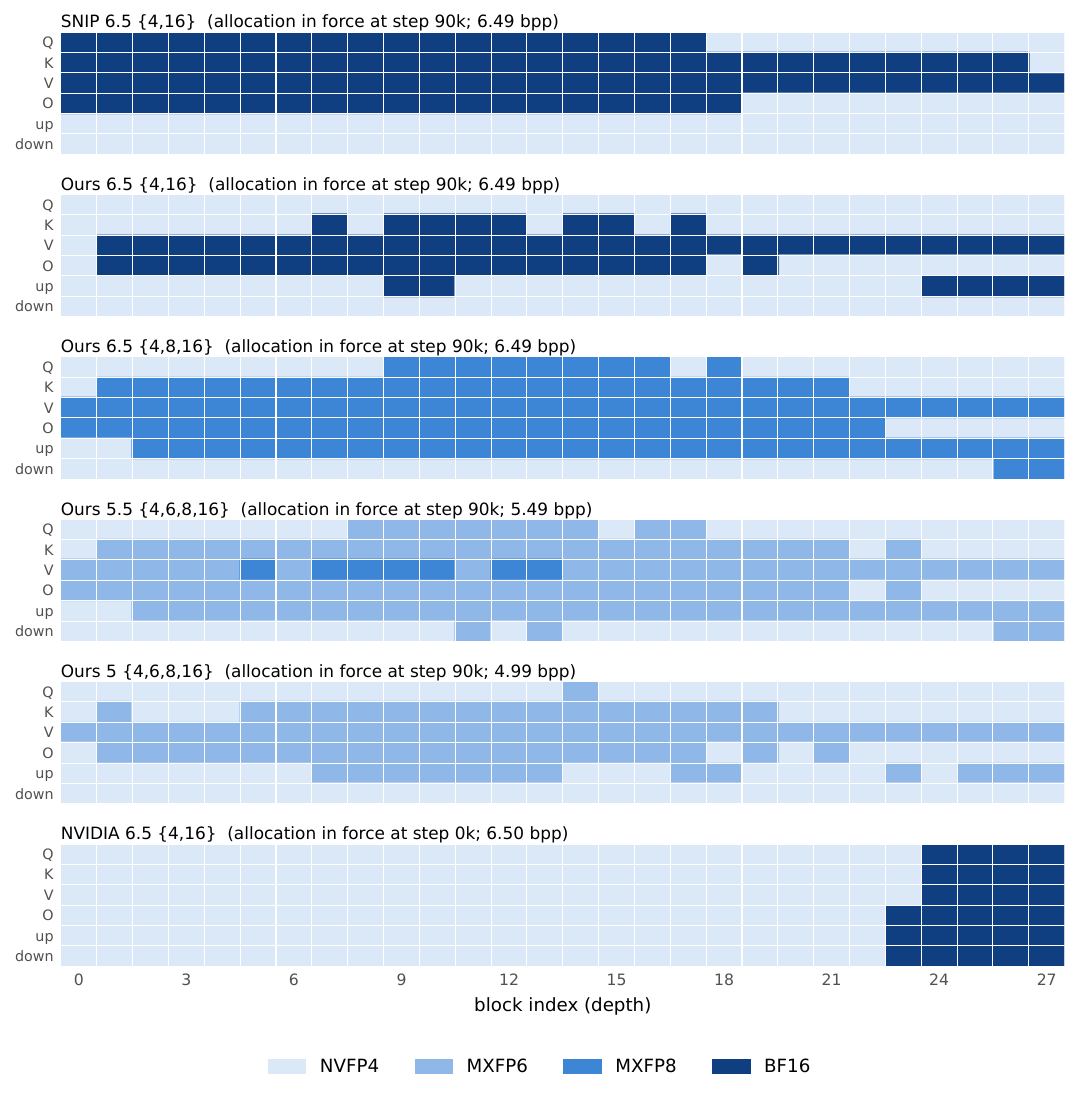}
\caption{Final allocations}
\end{subfigure}\\[6pt]
\begin{subfigure}{\textwidth}\centering
\includegraphics[width=0.9\linewidth]{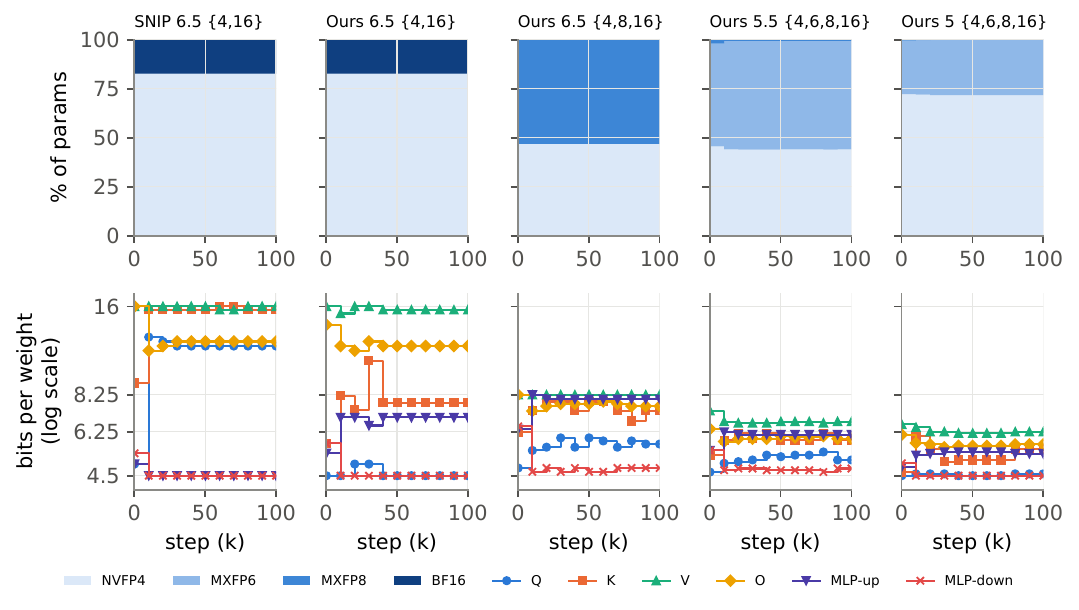}
\caption{Allocation dynamics}
\end{subfigure}
\caption{4B final allocations and their dynamics. Parameter share per format and bits per weight by projection type.}
\label{fig:alloc4b}
\end{figure}

\begin{figure}[htbp]
\centering
\includegraphics[width=0.95\textwidth]{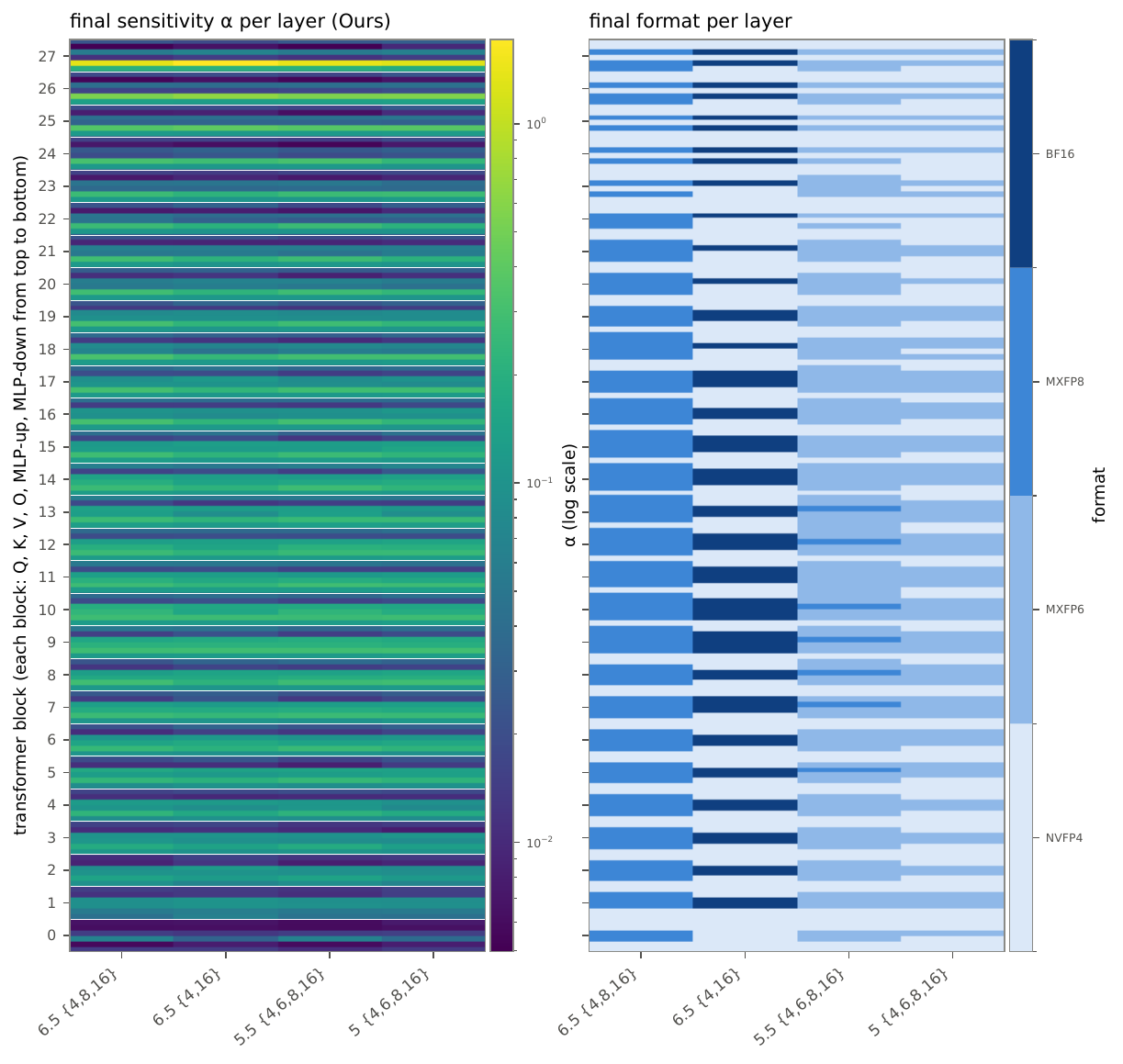}
\caption{Final sensitivity (logarithmic color scale) and format of all 168 switchable layers in the four 4B \methodname runs.}
\label{fig:heatall}
\end{figure}

\clearpage
\subsection{Downstream task scores}
\label{app:sft}
Accuracies are reported as percentages at the trained allocation and are averaged over three random seeds. The downstream mean presented in Table~\ref{tab:sftgrid} represents the average performance across 24 distinct tasks, with Natural Questions counted once. All tasks are evaluated in a zero-shot setting, except for GSM8K, which is assessed using an eight-shot chain-of-thought approach. Evaluation is conducted using the LM Evaluation Harness \citep{lmeval}.

\begin{itemize}
\item Generation and instruction following (Table~\ref{tab:sftgen}): GSM8K \citep{gsm8k}, exact match; IFEval \citep{ifeval}, strict prompt-level accuracy; HumanEval \citep{humaneval}, pass@1; Natural Questions \citep{nq}, exact match; TruthfulQA \citep{truthfulqa}, MC2 accuracy.
\item Knowledge and likelihood (Table~\ref{tab:sftraw}): ARC-Challenge, ARC-Easy, HellaSwag, PIQA, WinoGrande, LAMBADA, and MMLU, accuracy.
\item Probe tasks (Table~\ref{tab:probe1}): BoolQ, COPA, RACE, ASDiv, and CommonsenseQA, accuracy; LogiQA, SAT-Math, SciQ, OpenBookQA, AQuA-RAT, LSAT-AR, and LSAT-LR, length-normalized accuracy; Natural Questions, exact match.
\end{itemize}
\begin{table}[htbp]
\centering\scriptsize
\renewcommand{\arraystretch}{0.95}
\caption{SFT generative and instruction benchmarks (accuracy \%, mean $\pm$ seed standard deviation over three seeds).}
\label{tab:sftgen}
\setlength{\tabcolsep}{2.4pt}
\resizebox{\textwidth}{!}{%
\begin{tabular}{lccccc}
\toprule
run & GSM8K & IFEval & HumanEval & NQ & TQA \\
\midrule
\multicolumn{6}{l}{\textit{uniform references}} \\
uniform NVFP4 & $80.6_{\pm 1.4}$ & $30.6_{\pm 0.8}$ & $69.1_{\pm 2.9}$ & $8.7_{\pm 0.4}$ & $51.2_{\pm 0.7}$ \\
uniform MXFP8 & $84.2_{\pm 0.3}$ & $34.3_{\pm 1.1}$ & $75.8_{\pm 0.9}$ & $11.2_{\pm 0.7}$ & $52.6_{\pm 0.1}$ \\
\midrule
\multicolumn{6}{l}{\textit{budget 5, menu $\{4,6,8,16\}$}} \\
NVIDIA & $81.0_{\pm 0.9}$ & $28.0_{\pm 1.7}$ & $71.3_{\pm 2.7}$ & $9.2_{\pm 0.5}$ & $\mathbf{51.2}_{\pm 0.2}$ \\
\methodname & $81.1_{\pm 0.7}$ & $33.0_{\pm 2.2}$ & $\mathbf{72.8}_{\pm 3.1}$ & $8.9_{\pm 0.6}$ & $49.9_{\pm 0.6}$ \\
SNIP & $81.0_{\pm 0.9}$ & $\mathbf{33.8}_{\pm 0.7}$ & $71.1_{\pm 2.7}$ & $9.3_{\pm 0.4}$ & $50.0_{\pm 0.3}$ \\
\midrule
\multicolumn{6}{l}{\textit{budget 5.5, menu $\{4,6,8,16\}$}} \\
NVIDIA & $80.5_{\pm 0.5}$ & $30.2_{\pm 0.9}$ & $69.5_{\pm 1.1}$ & $9.2_{\pm 0.4}$ & $\mathbf{50.8}_{\pm 0.4}$ \\
\methodname & $\mathbf{84.3}_{\pm 0.8}$ & $32.2_{\pm 0.7}$ & $\mathbf{73.6}_{\pm 1.3}$ & $10.0_{\pm 0.7}$ & $50.6_{\pm 1.0}$ \\
SNIP & $83.5_{\pm 1.2}$ & $33.9_{\pm 0.3}$ & $70.7_{\pm 1.2}$ & $9.4_{\pm 0.5}$ & $50.8_{\pm 0.5}$ \\
\midrule
\multicolumn{6}{l}{\textit{budget 6.5, menu $\{4,8,16\}$}} \\
NVIDIA & $78.1_{\pm 1.0}$ & $28.0_{\pm 1.1}$ & $71.3_{\pm 1.1}$ & $9.7_{\pm 0.2}$ & $50.7_{\pm 0.2}$ \\
\methodname & $\mathbf{84.4}_{\pm 0.9}$ & $32.3_{\pm 1.7}$ & $73.4_{\pm 0.9}$ & $\mathbf{10.0}_{\pm 0.6}$ & $\mathbf{51.1}_{\pm 0.6}$ \\
SNIP & $83.1_{\pm 0.8}$ & $\mathbf{33.1}_{\pm 0.8}$ & $\mathbf{73.6}_{\pm 0.4}$ & $8.2_{\pm 0.3}$ & $51.0_{\pm 0.4}$ \\
\midrule
\multicolumn{6}{l}{\textit{budget 6.5, menu $\{4,16\}$}} \\
NVIDIA & $78.1_{\pm 1.0}$ & $28.0_{\pm 1.1}$ & $71.3_{\pm 1.1}$ & $\mathbf{9.7}_{\pm 0.2}$ & $\mathbf{50.7}_{\pm 0.2}$ \\
\methodname & $79.6_{\pm 0.9}$ & $33.3_{\pm 1.1}$ & $72.6_{\pm 1.2}$ & $8.8_{\pm 0.7}$ & $50.2_{\pm 0.6}$ \\
SNIP & $80.0_{\pm 0.6}$ & $33.1_{\pm 0.4}$ & $72.4_{\pm 1.3}$ & $9.3_{\pm 0.5}$ & $50.2_{\pm 0.3}$ \\
\bottomrule
\end{tabular}}
\end{table}

\begin{table}[htbp]
\centering\scriptsize
\renewcommand{\arraystretch}{0.95}
\caption{SFT knowledge/likelihood tasks (accuracy \%, mean $\pm$ seed standard deviation over three seeds).}
\label{tab:sftraw}
\setlength{\tabcolsep}{2.4pt}
\resizebox{\textwidth}{!}{%
\begin{tabular}{lccccccc}
\toprule
run & ARC-c & ARC-e & HellaSw & PIQA & Wino & LAMBADA & MMLU \\
\midrule
\multicolumn{8}{l}{\textit{uniform references}} \\
uniform NVFP4 & $44.9_{\pm 0.6}$ & $74.6_{\pm 0.3}$ & $53.8_{\pm 0.1}$ & $76.1_{\pm 0.4}$ & $67.6_{\pm 0.8}$ & $66.3_{\pm 0.2}$ & $68.7_{\pm 0.1}$ \\
uniform MXFP8 & $47.0_{\pm 0.2}$ & $78.4_{\pm 0.1}$ & $54.6_{\pm 0.0}$ & $78.0_{\pm 0.1}$ & $70.7_{\pm 0.6}$ & $68.6_{\pm 0.1}$ & $71.1_{\pm 0.1}$ \\
\midrule
\multicolumn{8}{l}{\textit{budget 5, menu $\{4,6,8,16\}$}} \\
NVIDIA & $44.6_{\pm 0.5}$ & $74.6_{\pm 0.7}$ & $53.7_{\pm 0.2}$ & $76.1_{\pm 0.3}$ & $67.9_{\pm 1.2}$ & $67.0_{\pm 0.2}$ & $68.8_{\pm 0.1}$ \\
\methodname & $46.3_{\pm 0.7}$ & $76.9_{\pm 0.1}$ & $53.9_{\pm 0.1}$ & $76.3_{\pm 0.3}$ & $68.4_{\pm 0.2}$ & $\mathbf{67.2}_{\pm 0.2}$ & $69.5_{\pm 0.1}$ \\
SNIP & $\mathbf{47.0}_{\pm 0.8}$ & $\mathbf{77.1}_{\pm 0.6}$ & $\mathbf{54.0}_{\pm 0.1}$ & $\mathbf{76.7}_{\pm 0.4}$ & $68.3_{\pm 0.7}$ & $65.8_{\pm 0.3}$ & $69.9_{\pm 0.1}$ \\
\midrule
\multicolumn{8}{l}{\textit{budget 5.5, menu $\{4,6,8,16\}$}} \\
NVIDIA & $45.6_{\pm 0.8}$ & $75.4_{\pm 0.6}$ & $53.6_{\pm 0.1}$ & $76.2_{\pm 0.3}$ & $67.4_{\pm 0.6}$ & $67.3_{\pm 0.3}$ & $68.7_{\pm 0.2}$ \\
\methodname & $\mathbf{46.8}_{\pm 0.6}$ & $77.1_{\pm 1.0}$ & $54.2_{\pm 0.0}$ & $77.0_{\pm 0.1}$ & $68.9_{\pm 0.7}$ & $\mathbf{67.9}_{\pm 0.2}$ & $70.1_{\pm 0.1}$ \\
SNIP & $46.2_{\pm 0.5}$ & $77.5_{\pm 0.4}$ & $54.2_{\pm 0.0}$ & $\mathbf{77.2}_{\pm 0.3}$ & $\mathbf{69.1}_{\pm 0.3}$ & $67.4_{\pm 0.3}$ & $\mathbf{70.5}_{\pm 0.1}$ \\
\midrule
\multicolumn{8}{l}{\textit{budget 6.5, menu $\{4,8,16\}$}} \\
NVIDIA & $44.8_{\pm 0.2}$ & $74.6_{\pm 0.6}$ & $53.8_{\pm 0.1}$ & $76.6_{\pm 0.3}$ & $68.5_{\pm 1.5}$ & $67.2_{\pm 0.2}$ & $68.9_{\pm 0.1}$ \\
\methodname & $47.0_{\pm 0.2}$ & $77.3_{\pm 0.5}$ & $54.1_{\pm 0.1}$ & $\mathbf{77.1}_{\pm 0.1}$ & $68.8_{\pm 0.7}$ & $\mathbf{68.0}_{\pm 0.2}$ & $70.0_{\pm 0.0}$ \\
SNIP & $46.5_{\pm 0.3}$ & $76.7_{\pm 0.2}$ & $\mathbf{54.3}_{\pm 0.2}$ & $76.6_{\pm 0.2}$ & $68.9_{\pm 0.7}$ & $67.6_{\pm 0.0}$ & $\mathbf{70.4}_{\pm 0.3}$ \\
\midrule
\multicolumn{8}{l}{\textit{budget 6.5, menu $\{4,16\}$}} \\
NVIDIA & $44.8_{\pm 0.2}$ & $74.6_{\pm 0.6}$ & $53.8_{\pm 0.1}$ & $76.6_{\pm 0.3}$ & $68.5_{\pm 1.5}$ & $67.2_{\pm 0.2}$ & $68.9_{\pm 0.1}$ \\
\methodname & $46.4_{\pm 0.6}$ & $76.5_{\pm 1.0}$ & $53.9_{\pm 0.0}$ & $\mathbf{76.9}_{\pm 0.3}$ & $\mathbf{69.0}_{\pm 0.4}$ & $67.0_{\pm 0.1}$ & $69.7_{\pm 0.2}$ \\
SNIP & $\mathbf{47.3}_{\pm 0.8}$ & $\mathbf{77.3}_{\pm 0.2}$ & $54.0_{\pm 0.2}$ & $76.5_{\pm 0.3}$ & $67.8_{\pm 0.6}$ & $67.0_{\pm 0.1}$ & $\mathbf{70.1}_{\pm 0.2}$ \\
\bottomrule
\end{tabular}}
\end{table}

\begin{table}[htbp]
\centering\scriptsize
\renewcommand{\arraystretch}{0.95}
\caption{SFT probe tasks (accuracy \%, mean $\pm$ seed standard deviation over three seeds).}
\label{tab:probe1}
\setlength{\tabcolsep}{2.4pt}
\begin{tabular}{lccccccc}
\toprule
run & BoolQ & Copa & RACE & LogiQA & ASDiv & SAT-Math & NQ \\
\midrule
\multicolumn{8}{l}{\textit{uniform references}} \\
uniform NVFP4 & $77.1_{\pm 0.2}$ & $84.0_{\pm 1.7}$ & $39.5_{\pm 0.7}$ & $40.2_{\pm 0.7}$ & $27.9_{\pm 0.2}$ & $43.6_{\pm 2.0}$ & $8.7_{\pm 0.4}$ \\
uniform MXFP8 & $81.2_{\pm 0.3}$ & $83.0_{\pm 0.0}$ & $40.5_{\pm 0.1}$ & $42.8_{\pm 0.6}$ & $26.6_{\pm 0.5}$ & $48.5_{\pm 2.1}$ & $11.2_{\pm 0.7}$ \\
\midrule
\multicolumn{8}{l}{\textit{budget 5, menu $\{4,6,8,16\}$}} \\
NVIDIA & $79.3_{\pm 0.4}$ & $83.7_{\pm 0.6}$ & $\mathbf{40.4}_{\pm 0.9}$ & $39.6_{\pm 0.8}$ & $27.4_{\pm 0.9}$ & $42.7_{\pm 1.4}$ & $9.2_{\pm 0.5}$ \\
\methodname & $\mathbf{81.9}_{\pm 0.6}$ & $83.7_{\pm 3.1}$ & $40.1_{\pm 0.9}$ & $41.0_{\pm 0.8}$ & $28.5_{\pm 1.0}$ & $45.8_{\pm 1.3}$ & $8.9_{\pm 0.6}$ \\
SNIP & $80.6_{\pm 0.5}$ & $83.7_{\pm 1.2}$ & $39.5_{\pm 0.5}$ & $\mathbf{42.3}_{\pm 1.7}$ & $\mathbf{31.9}_{\pm 1.1}$ & $\mathbf{46.1}_{\pm 2.7}$ & $9.3_{\pm 0.4}$ \\
\midrule
\multicolumn{8}{l}{\textit{budget 5.5, menu $\{4,6,8,16\}$}} \\
NVIDIA & $79.0_{\pm 0.5}$ & $84.3_{\pm 2.1}$ & $\mathbf{40.5}_{\pm 0.4}$ & $40.8_{\pm 0.8}$ & $26.6_{\pm 0.9}$ & $43.5_{\pm 1.1}$ & $9.2_{\pm 0.4}$ \\
\methodname & $\mathbf{82.2}_{\pm 0.5}$ & $84.0_{\pm 2.6}$ & $40.4_{\pm 0.1}$ & $\mathbf{42.0}_{\pm 0.2}$ & $32.6_{\pm 1.8}$ & $46.2_{\pm 2.1}$ & $10.0_{\pm 0.7}$ \\
SNIP & $80.0_{\pm 0.2}$ & $\mathbf{86.7}_{\pm 2.1}$ & $40.3_{\pm 0.2}$ & $40.6_{\pm 0.2}$ & $31.7_{\pm 0.1}$ & $\mathbf{48.8}_{\pm 2.3}$ & $9.4_{\pm 0.5}$ \\
\midrule
\multicolumn{8}{l}{\textit{budget 6.5, menu $\{4,8,16\}$}} \\
NVIDIA & $79.2_{\pm 0.8}$ & $\mathbf{85.0}_{\pm 1.0}$ & $39.6_{\pm 0.4}$ & $39.6_{\pm 1.9}$ & $25.7_{\pm 0.7}$ & $43.2_{\pm 2.1}$ & $9.7_{\pm 0.2}$ \\
\methodname & $\mathbf{82.0}_{\pm 0.9}$ & $84.3_{\pm 0.6}$ & $40.4_{\pm 0.3}$ & $40.3_{\pm 1.5}$ & $31.2_{\pm 2.0}$ & $45.6_{\pm 4.4}$ & $\mathbf{10.0}_{\pm 0.6}$ \\
SNIP & $79.3_{\pm 0.5}$ & $84.0_{\pm 1.0}$ & $39.8_{\pm 0.2}$ & $41.6_{\pm 1.6}$ & $\mathbf{33.1}_{\pm 1.1}$ & $\mathbf{48.9}_{\pm 0.9}$ & $8.2_{\pm 0.3}$ \\
\midrule
\multicolumn{8}{l}{\textit{budget 6.5, menu $\{4,16\}$}} \\
NVIDIA & $79.2_{\pm 0.8}$ & $\mathbf{85.0}_{\pm 1.0}$ & $39.6_{\pm 0.4}$ & $39.6_{\pm 1.9}$ & $25.7_{\pm 0.7}$ & $43.2_{\pm 2.1}$ & $\mathbf{9.7}_{\pm 0.2}$ \\
\methodname & $80.1_{\pm 0.4}$ & $83.7_{\pm 1.2}$ & $40.1_{\pm 0.6}$ & $\mathbf{41.0}_{\pm 0.9}$ & $29.3_{\pm 1.5}$ & $45.2_{\pm 2.2}$ & $8.8_{\pm 0.7}$ \\
SNIP & $80.6_{\pm 0.2}$ & $84.3_{\pm 2.5}$ & $\mathbf{40.7}_{\pm 0.6}$ & $40.3_{\pm 1.3}$ & $\mathbf{33.8}_{\pm 0.9}$ & $45.2_{\pm 1.6}$ & $9.3_{\pm 0.5}$ \\
\bottomrule
\end{tabular}

\vspace{8pt}
\begin{tabular}{lcccccc}
\toprule
run & SciQ & OBQA & CSQA & AQuA & LSAT-ar & LSAT-lr \\
\midrule
\multicolumn{7}{l}{\textit{uniform references}} \\
uniform NVFP4 & $94.2_{\pm 0.4}$ & $41.2_{\pm 0.9}$ & $79.2_{\pm 0.8}$ & $39.5_{\pm 1.6}$ & $18.0_{\pm 1.8}$ & $43.1_{\pm 1.2}$ \\
uniform MXFP8 & $94.5_{\pm 0.2}$ & $40.3_{\pm 0.6}$ & $82.1_{\pm 0.5}$ & $40.6_{\pm 2.2}$ & $19.4_{\pm 0.5}$ & $48.6_{\pm 0.2}$ \\
\midrule
\multicolumn{7}{l}{\textit{budget 5, menu $\{4,6,8,16\}$}} \\
NVIDIA & $94.4_{\pm 0.1}$ & $40.7_{\pm 0.3}$ & $79.0_{\pm 0.6}$ & $\mathbf{39.8}_{\pm 1.7}$ & $\mathbf{21.2}_{\pm 1.8}$ & $39.8_{\pm 0.9}$ \\
\methodname & $94.6_{\pm 0.2}$ & $40.7_{\pm 0.7}$ & $\mathbf{79.9}_{\pm 0.8}$ & $36.7_{\pm 2.0}$ & $20.6_{\pm 3.1}$ & $43.8_{\pm 1.6}$ \\
SNIP & $\mathbf{94.7}_{\pm 0.3}$ & $\mathbf{41.8}_{\pm 0.2}$ & $79.8_{\pm 0.2}$ & $38.2_{\pm 1.6}$ & $21.0_{\pm 0.9}$ & $\mathbf{44.4}_{\pm 0.8}$ \\
\midrule
\multicolumn{7}{l}{\textit{budget 5.5, menu $\{4,6,8,16\}$}} \\
NVIDIA & $94.4_{\pm 0.1}$ & $\mathbf{41.2}_{\pm 0.7}$ & $78.7_{\pm 0.3}$ & $38.7_{\pm 2.8}$ & $18.0_{\pm 1.6}$ & $39.7_{\pm 1.0}$ \\
\methodname & $\mathbf{94.7}_{\pm 0.1}$ & $40.5_{\pm 0.3}$ & $79.6_{\pm 0.6}$ & $38.5_{\pm 0.6}$ & $\mathbf{19.3}_{\pm 3.3}$ & $42.4_{\pm 1.6}$ \\
SNIP & $94.3_{\pm 0.2}$ & $41.0_{\pm 0.3}$ & $80.1_{\pm 0.4}$ & $37.5_{\pm 1.8}$ & $17.8_{\pm 1.2}$ & $\mathbf{44.2}_{\pm 3.1}$ \\
\midrule
\multicolumn{7}{l}{\textit{budget 6.5, menu $\{4,8,16\}$}} \\
NVIDIA & $94.4_{\pm 0.5}$ & $40.3_{\pm 0.2}$ & $78.1_{\pm 0.6}$ & $38.7_{\pm 0.8}$ & $18.6_{\pm 0.7}$ & $40.4_{\pm 0.4}$ \\
\methodname & $\mathbf{94.9}_{\pm 0.3}$ & $40.9_{\pm 0.3}$ & $79.7_{\pm 0.2}$ & $38.1_{\pm 1.3}$ & $19.4_{\pm 0.7}$ & $44.9_{\pm 1.2}$ \\
SNIP & $94.2_{\pm 0.2}$ & $\mathbf{41.4}_{\pm 0.5}$ & $\mathbf{80.6}_{\pm 1.0}$ & $\mathbf{41.1}_{\pm 1.8}$ & $\mathbf{21.2}_{\pm 1.5}$ & $\mathbf{46.0}_{\pm 0.8}$ \\
\midrule
\multicolumn{7}{l}{\textit{budget 6.5, menu $\{4,16\}$}} \\
NVIDIA & $94.4_{\pm 0.5}$ & $40.3_{\pm 0.2}$ & $78.1_{\pm 0.6}$ & $38.7_{\pm 0.8}$ & $18.6_{\pm 0.7}$ & $40.4_{\pm 0.4}$ \\
\methodname & $94.5_{\pm 0.2}$ & $40.7_{\pm 1.0}$ & $79.7_{\pm 0.6}$ & $37.8_{\pm 3.1}$ & $20.9_{\pm 0.9}$ & $\mathbf{43.5}_{\pm 0.4}$ \\
SNIP & $94.7_{\pm 0.2}$ & $41.0_{\pm 1.6}$ & $\mathbf{80.1}_{\pm 0.8}$ & $36.9_{\pm 0.9}$ & $19.6_{\pm 2.6}$ & $43.0_{\pm 0.8}$ \\
\bottomrule
\end{tabular}
\end{table}

\end{document}